\documentclass{article}

\usepackage[preprint]{cpal_2026}

\usepackage{url}
\usepackage{amsmath}
\usepackage{amssymb}
\usepackage{booktabs}
\usepackage{multirow}
\usepackage{graphicx}
\usepackage{amsthm}
\usepackage{enumitem}
\usepackage{float}  
\usepackage{placeins}  

\newtheorem{definition}{Definition}
\newtheorem{proposition}{Proposition}
\newtheorem{lemma}{Lemma}

\newtheorem{corollary}{Corollary}
\theoremstyle{remark}
\newtheorem{remark}{Remark}

\title{SPIKE: Sparse Koopman Regularization for Physics-Informed Neural Networks}

\author{%
  Jose Marie Antonio Mi\~noza \\
  Center for AI Research PH
}

\begin{document}

\maketitle

\begin{abstract}
Physics-Informed Neural Networks (PINNs) provide a mesh-free approach for solving differential equations by embedding physical constraints into neural network training. However, PINNs tend to overfit within the training domain, leading to poor generalization when extrapolating beyond trained spatiotemporal regions. This work presents SPIKE (Sparse Physics-Informed Koopman-Enhanced), a framework that regularizes PINNs with continuous-time Koopman operators to learn parsimonious dynamics representations. By enforcing linear dynamics $dz/dt = Az$ in a learned observable space, both PIKE (without explicit sparsity) and SPIKE (with L1 regularization on $A$) learn sparse generator matrices, embodying the parsimony principle that complex dynamics admit low-dimensional structure. Experiments across parabolic, hyperbolic, dispersive, and stiff PDEs, including fluid dynamics (Navier-Stokes) and chaotic ODEs (Lorenz), demonstrate consistent improvements in temporal extrapolation, spatial generalization, and long-term prediction accuracy. The continuous-time formulation with matrix exponential integration provides unconditional stability for stiff systems while avoiding diagonal dominance issues inherent in discrete-time Koopman operators.
\end{abstract}

\section{Introduction}

Solving partial differential equations with reliable out-of-distribution (OOD) generalization remains a fundamental challenge in scientific computing \citep{wang2022pinnsfail}. Physics-Informed Neural Networks (PINNs), introduced by \citet{raissi2019pinn}, provide a mesh-free approach by embedding differential equation residuals directly into neural network training. Given a partial differential equation $\mathcal{N}[u] = f$ where $\mathcal{N}$ is a nonlinear differential operator, PINNs minimize a physics loss:
\begin{equation}
\mathcal{L}_{\text{physics}} = \frac{1}{N_r}\sum_{i=1}^{N_r}\left|\mathcal{N}[u_\theta](x_i, t_i) - f(x_i, t_i)\right|^2
\end{equation}
alongside initial and boundary condition losses. This formulation enables mesh-free solutions with minimal data requirements \citep{raissi2019pinn, wang2021understanding}. However, PINNs face two key limitations: (1) the neural network representation $u_\theta$ remains opaque: the learned weights do not directly reveal the underlying dynamics, limiting interpretability; and (2) without structural regularization, PINNs tend to overfit within the training domain, leading to rapid error growth when extrapolating beyond trained time horizons \citep{wang2022pinnsfail}.

Koopman operator theory \citep{koopman1931hamiltonian} provides a complementary framework where nonlinear dynamical systems admit linear representations in infinite-dimensional observable spaces. For an autonomous system $\dot{x} = f(x)$, the Koopman operator family $\mathcal{K}^t$ acts on observables $g: \mathcal{X} \rightarrow \mathbb{C}$ via $\mathcal{K}^t g(x) = g(F^t(x))$, where $F^t$ denotes the flow map. The infinitesimal generator $\mathcal{L}$, also known as the Lie operator, satisfies the fundamental relation \citep{brunton2022modern}:
\begin{equation}
\mathcal{L}g = \nabla g \cdot f
\label{eq:lie_operator}
\end{equation}
In finite-dimensional approximations, this yields linear dynamics $dz/dt = Az$ where $z = g(x)$ represents lifted observables and $A \in \mathbb{R}^{M \times M}$ approximates the generator.

The present work proposes a framework where the PINN remains the base model for solving PDEs via automatic differentiation, with the Koopman component serving as an auxiliary regularizer that promotes sparse, interpretable structure in the learned dynamics. This inverts the typical paradigm: rather than augmenting Koopman methods with physics constraints, the approach enhances PINNs with Koopman regularization.

\textbf{Nomenclature.} \textbf{PIKE} (Physics-Informed Koopman-Enhanced) denotes the base framework with Koopman regularization. \textbf{SPIKE} (Sparse PIKE) adds explicit L1 sparsity on the generator matrix $A$. Both learn sparse structure; SPIKE enforces it explicitly.

\textbf{Contributions.} This work makes the following contributions:
\begin{enumerate}
\item \textbf{Koopman regularization for generalization}: The sparse Koopman constraint acts as an implicit regularizer, improving temporal extrapolation by 2--184$\times$ and preventing catastrophic failure outside training regions. For open-domain systems, this enables physically meaningful spatial generalization.

\item \textbf{Parsimonious dynamics representation}: L1 sparsity reduces non-zero generator entries by up to 5.7$\times$, yielding parsimonious representations embodying the parsimony principle \citep{ma2022principles}. For inverse problems, sparse structure enables hypothesis generation.

\item \textbf{Library-latent decomposition}: A dual-component observable embedding combining explicit polynomial terms with learned MLP features (up to 0.99 correlation with $u_{xx}$), enabling both coefficient recovery and interpretability.

\item \textbf{Continuous-time Koopman formulation}: Direct learning of the generator $A$ via $dz/dt = Az$, avoiding diagonal dominance issues in discrete-time formulations.
\end{enumerate}

\section{Related Work}
\label{sec:related}

\textbf{Physics-Informed Neural Networks.} The PINN framework \citep{raissi2019pinn} approximates solutions to differential equations using neural networks trained with physics-based loss functions. For a PDE $u_t + \mathcal{N}[u] = 0$, a neural network $u_\theta(x, t)$ minimizes $\mathcal{L} = \lambda_r\mathcal{L}_{\text{physics}} + \lambda_{ic}\mathcal{L}_{IC} + \lambda_{bc}\mathcal{L}_{BC}$ where automatic differentiation computes required derivatives. PINNs have been applied to Navier-Stokes equations \citep{raissi2019pinn}, turbulence modeling \citep{wang2021understanding}, and inverse problems \citep{raissi2020hidden}. However, PINNs can fail to train due to spectral bias and gradient pathologies \citep{wang2022pinnsfail}, and tend to overfit within the training domain. Extensions address training difficulties through gradient balancing \citep{wang2021understanding}, loss modifications, and architectural innovations; \citet{hanna2025variance} proposed variance-based regularization, while AC-PKAN \citep{zhang2025acpkan} combines Chebyshev-based Kolmogorov-Arnold Networks with attention mechanisms. Such approaches improve training stability or expressiveness but do not address OOD generalization---PIKE/SPIKE is orthogonal and could be combined with these architectures. Neural operators (FNO \citep{li2021fno}, DeepONet \citep{lu2021deeponet}) learn function-space mappings but also exhibit OOD limitations \citep{zhang2024dualbranchOOD, wu2025temporaloperator}.

\textbf{Koopman Operator Theory.} The Koopman operator provides a linear representation of nonlinear dynamics in an infinite-dimensional function space \citep{brunton2022modern}. Dynamic Mode Decomposition (DMD) \citep{rowley2009spectral} and Extended DMD \citep{williams2015edmd} approximate this operator from data. Deep Koopman approaches include autoencoders for invariant subspaces \citep{lusch2018deep, azencot2020forecasting} and Physics-Informed Koopman Networks (PIKN) \citep{liu2022physicskoopman}, which incorporate the Lie operator constraint $\mathcal{L}g = \nabla g \cdot f$. Unlike PIKN, which requires explicit knowledge of the dynamics $f(x)$, PIKE operates directly through the PDE residual. The Koopman Regularization framework of \citet{cohen2025koopman} proves that for an $N$-dimensional system, exactly $N$ functionally independent Koopman eigenfunctions suffice, justifying sparse $A$ matrices. However, while Koopman Regularization targets system identification, PIKE/SPIKE targets OOD generalization for PDE solving.

\textbf{Sparse Dynamics Discovery.} SINDy \citep{brunton2016sindy} pioneered sparse regression using finite-difference derivatives. Extensions include PDE-FIND \citep{rudy2017pdefind}, neural hybrids \citep{champion2019discovery}, and mesh-free variants \citep{meshfreesindy2025}. SPIKE addresses their noise sensitivity via autograd derivatives from the trained PINN.

\textbf{Key Distinctions.} Table~\ref{tab:related_comparison} summarizes methodological differences. PIKE/SPIKE: (1) operates on PDE residuals, not explicit $f(x)$; (2) targets OOD generalization, not training stability \citep{hanna2025variance} or system identification; (3) employs continuous-time $dz/dt = Az$ with multiple integrators (Euler/RK4/expm); (4) enables post-hoc sparse discovery without retraining.

\section{Theoretical Foundation}
\label{sec:theory}

This section establishes the theoretical basis connecting the Koopman infinitesimal generator to physics-informed learning. The key insight, following \citet{liu2022physicskoopman}, is that the Lie operator relation provides a principled mechanism for enforcing dynamical consistency.

\subsection{Koopman Generator and Physics Constraints}

\begin{definition}[Koopman Operator Family]
\label{def:koopman}
For an autonomous dynamical system $\frac{d}{dt}x(t) = f(x(t))$ with $x \in \mathcal{X} \subseteq \mathbb{R}^n$, let $F^t: \mathcal{X} \rightarrow \mathcal{X}$ denote the time-$t$ flow map satisfying $x(t_0 + t) = F^t(x(t_0))$. The Koopman operator family $\mathcal{K}^t: \mathcal{G}(\mathcal{X}) \rightarrow \mathcal{G}(\mathcal{X})$ acts on observables $g: \mathcal{X} \rightarrow \mathbb{C}$ as:
\begin{equation}
\mathcal{K}^t g(x) = g(F^t(x))
\end{equation}
\end{definition}

\begin{proposition}[Lie Operator Consistency]
\label{prop:lie}
For an autonomous system $\dot{x} = f(x)$ and differentiable observable $g: \mathcal{X} \rightarrow \mathbb{C}$, the infinitesimal generator $\mathcal{L}$ (Lie operator) satisfies:
\begin{equation}
\mathcal{L}g = \nabla g \cdot f
\label{eq:lie_fundamental}
\end{equation}
(Proof in Appendix~\ref{app:proofs}.)
\end{proposition}

\begin{proposition}[Finite-Dimensional Approximation]
\label{prop:finite_dim}
Let $g = [g_1, \ldots, g_M]^T$ be a vector of observable functions spanning a finite-dimensional subspace, and let $A \in \mathbb{R}^{M \times M}$ approximate the generator restricted to this subspace. The physics-informed Koopman loss:
\begin{equation}
\mathcal{L}_{\text{Lie}} = \|Ag(x) - \nabla g(x) \cdot f(x)\|^2
\label{eq:lie_loss}
\end{equation}
enforces consistency between the linear dynamics $\dot{z} = Az$ (where $z = g(x)$) and the true nonlinear system. (Proof in Appendix~\ref{app:proofs}.)
\end{proposition}

\begin{remark}[Approximation Error Bounds]
Proposition~\ref{prop:finite_dim} establishes existence of finite-dimensional approximations; general error bounds for neural network observables remain open. This work focuses on empirical validation following standard practice in physics-informed learning \citep{raissi2019pinn}.
\end{remark}

\begin{proposition}[Sparsity and Polynomial Representation]
\label{prop:sparsity}
Under L1 regularization $\lambda_s\|A\|_1$, the minimizer of the combined loss tends toward sparse solutions. For polynomial observables $g = [1, u, u^2, uv, \ldots]^T$, non-zero entries $A_{ij}$ correspond to active polynomial terms, providing an interpretable representation of the dynamics. (Proof in Appendix~\ref{app:proofs}.)
\end{proposition}

\begin{remark}[Lie Loss as Regularizer]
Equation~\ref{eq:lie_loss} corresponds to the Lie operator consistency constraint used in physics-informed Koopman methods \citep{liu2022physicskoopman}. The key distinction is that SPIKE enforces this constraint as a regularizer within a PINN framework, rather than as the primary training objective.
\end{remark}

\begin{proposition}[Out-of-Distribution Generalization Bound]
\label{prop:ood_bound}
Let $u_\theta$ be a PINN solution trained on domain $\Omega \times [0, T]$ with Koopman consistency error $\epsilon_K = \mathbb{E}[\|Az - \dot{z}\|^2]^{1/2}$. Let $z = g(u_\theta)$ be the observable embedding. Under the following assumptions:
\begin{enumerate}[label=(\roman*)]
    \item Spectral bound: all eigenvalues of $A$ satisfy $\text{Re}(\lambda_i) \leq \rho_0$
    \item Lipschitz decoder: $\|g^{-1}(z_1) - g^{-1}(z_2)\| \leq L_g\|z_1 - z_2\|$
\end{enumerate}
the extrapolation error for $t \in [T, T + \delta]$ satisfies:
\begin{equation}
\|u_\theta(x, t) - u^*(x, t)\|_{L^2(\Omega)} \leq L_g \cdot \frac{\epsilon_K}{\rho_0}\left(e^{\rho_0 \delta} - 1\right) + \mathcal{O}(\epsilon_{\text{train}})
\label{eq:ood_bound}
\end{equation}
where $\epsilon_{\text{train}}$ is the training domain error and $u^*$ is the true solution. (Proof in Appendix~\ref{app:proofs}; follows from Gronwall's inequality.)
\end{proposition}

\begin{remark}[Interpretation of Bound]
Equation~\ref{eq:ood_bound} shows: (1) small Koopman loss $\epsilon_K$ directly reduces OOD error; (2) bounded spectral radius $\rho_0$ controls exponential growth; (3) linear error growth initially since $(e^{\rho_0\delta} - 1)/\rho_0 \approx \delta$ for small $\delta$. This bound extends to spatial extrapolation under Lipschitz continuity assumptions (Corollary~\ref{cor:spatial_ood}, Appendix~\ref{app:proofs}). The Lipschitz decoder assumption (ii) holds for neural networks with bounded weights and Lipschitz activations; for the 4-layer MLP with tanh activation used here, $L_g \leq \prod_l \|W_l\|$. Empirical estimation of $L_g$ remains an open direction.
\end{remark}

\section{Physics-Informed Koopman-Enhanced Neural Networks}
\label{sec:method}

\subsection{Architecture Overview}

The proposed architecture augments a standard PINN with a Koopman regularization branch. The neural network encoder maps inputs $(x, t)$ to the solution field $u(x, t)$, while a parallel embedding layer lifts $u$ to an observable space where linear dynamics are enforced. Figure~\ref{fig:architecture} illustrates the overall framework.

\begin{figure}[!htb]
\centering
\includegraphics[width=0.95\textwidth]{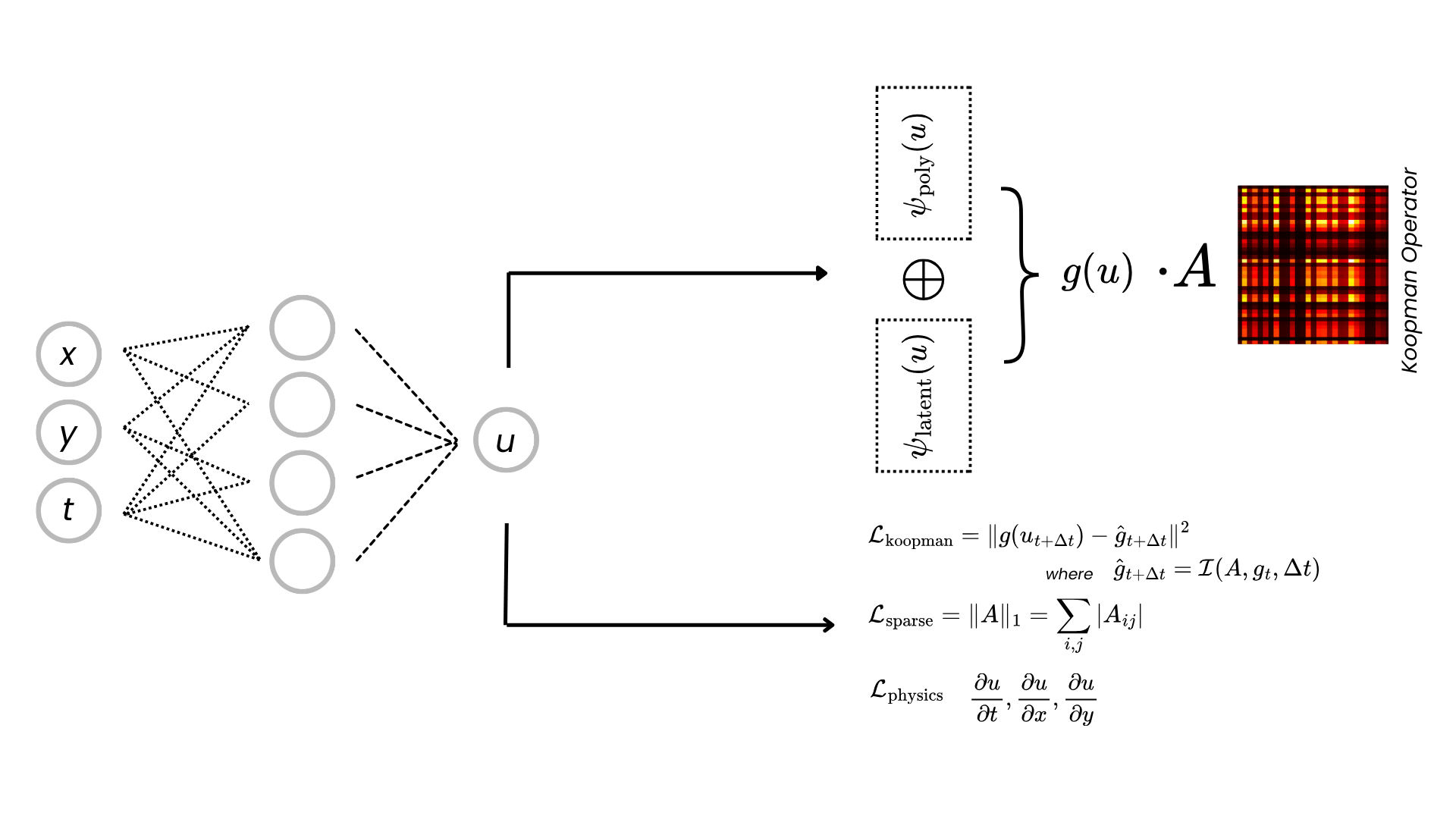}
\caption{PIKE/SPIKE architecture. A neural network maps coordinates $(x, y, t)$ to solution $u$, which is lifted to observables $g(u) = \psi_{\text{poly}}(u) \oplus \psi_{\text{latent}}(u)$. The Koopman operator $A$ governs linear dynamics $dz/dt = Az$. Three losses are minimized: physics residual $\mathcal{L}_{\text{physics}}$ via automatic differentiation, Koopman consistency $\mathcal{L}_{\text{koopman}}$, and L1 sparsity $\mathcal{L}_{\text{sparse}} = \|A\|_1$ (SPIKE only).}
\label{fig:architecture}
\end{figure}

\begin{definition}[Augmented Embedding]
\label{def:augmented}
The observable embedding $g: \mathbb{R}^n \to \mathbb{R}^M$ decomposes as $g(u) = [g_{\text{lib}}(u), g_{\text{mlp}}(u)]^T$ where:
\begin{equation}
g_{\text{lib}}(u) = W_{\text{lib}} \cdot \psi_d(u), \quad g_{\text{mlp}}(u) = \sigma(W_L \cdots \sigma(W_1 u))
\end{equation}
with $\psi_d(u) = [1, u_1, \ldots, u_n, u_1^2, u_1 u_2, \ldots]^T$ containing all monomials up to degree $d$, $W_{\text{lib}} \in \mathbb{R}^{M_{\text{lib}} \times \binom{n+d}{d}}$ a learnable projection, and $\{W_i\}$ parameterizing an $L$-layer MLP with activation $\sigma$.
\end{definition}

The embedding comprises: (1) a \textbf{polynomial library} $[1, u, u^2, \ldots]$ mapping to symbolic expressions, and (2) \textbf{learned MLP features} providing additional capacity. Polynomial observables follow from Koopman theory: polynomial systems admit finite-dimensional invariant subspaces \citep{brunton2016sindy}. Degree-2 captures quadratic nonlinearities common in physical systems. The library can be extended to include derivative terms $[u_x, u_{xx}, u \cdot u_x, \ldots]$ for explicit representation of convective and diffusive dynamics. The current experiments use polynomial-only libraries to evaluate the latent MLP's capacity for implicit derivative learning; results show high correlations (up to 0.99 with $u_{xx}$ for Heat; see Appendix~\ref{app:interpretability}), validating that derivative structure emerges in the latent space even without explicit library terms.

\subsection{Continuous-Time Koopman Formulation}

A critical design choice distinguishes continuous-time from discrete-time Koopman operators. The discrete formulation $z_{t+\Delta t} = K z_t$ suffers from diagonal dominance as $\Delta t \rightarrow 0$:
\begin{equation}
K = e^{A\Delta t} \approx I + A\Delta t + O(\Delta t^2)
\end{equation}
For small timesteps, $K$ approaches the identity matrix, with off-diagonal elements becoming negligible.

\begin{lemma}[Continuous Generator Advantage]
\label{lem:continuous}
Let $K = e^{A\Delta t}$ be the discrete Koopman operator and $A$ the infinitesimal generator. For $\Delta t \to 0$:
\begin{enumerate}
    \item[(i)] The discrete operator satisfies $\|K - I\|_F \leq \|A\|_F \Delta t + O(\Delta t^2)$, causing off-diagonal entries to vanish.
    \item[(ii)] The generator $A$ remains $\Delta t$-independent, with $A_{ij}$ ($i \neq j$) directly encoding the rate at which observable $g_j$ influences $\dot{g}_i$.
\end{enumerate}
(Proof in Appendix~\ref{app:proofs}.)
\end{lemma}

The continuous-time formulation directly learns the generator:
\begin{equation}
\frac{dz}{dt} = Az
\end{equation}
Here, $A$ is independent of $\Delta t$, and off-diagonal entries directly represent interaction strengths.

\subsection{Matrix Exponential Integration for Stiff Systems}

For stiff PDEs (Cahn-Hilliard, Kuramoto-Sivashinsky), finite-difference Koopman losses exhibit training instability. Following work on exponential integrators for neural ODEs \citep{jolicoeurmartineau2024stiffnode, schaeffer2025structurepreserving}, the framework adopts a matrix exponential formulation providing the \emph{exact} solution for linear Koopman dynamics:
\begin{equation}
z(t + \Delta t) = e^{A\Delta t} z(t)
\label{eq:expm_propagation}
\end{equation}

The corresponding loss directly compares propagated states:
\begin{equation}
\mathcal{L}_{\text{koopman}}^{\text{expm}} = \frac{1}{N}\sum_{i=1}^{N}\left\|e^{A\Delta t}z_i - z_{i+1}\right\|^2
\end{equation}

\begin{proposition}[Unconditional Stability of Matrix Exponential]
\label{prop:expm_stability}
The matrix exponential propagation (Equation~\ref{eq:expm_propagation}) is unconditionally A-stable: for any $\Delta t > 0$ and any matrix $A$ with eigenvalues having non-positive real parts, $\|e^{A\Delta t}z\| \leq \|z\|$.
\end{proposition}

The integrating factor Euler method $z_{n+1} = e^{A\Delta t}z_n$ has been shown to be the most reliable explicit exponential method for stiff neural ODEs \citep{garrido2024expintegrators}.

The matrix exponential is computed via Pad\'e approximation with scaling and squaring, ensuring both accuracy and differentiability for gradient-based training. For stiff PDEs where $|\lambda_{\max}|$ can exceed $10^4$, this formulation enables SPIKE to handle fourth-order PDEs and chaotic systems that would otherwise exhibit training collapse (detailed stability analysis in Appendix~\ref{app:detailed_results}).

\subsection{Sparsity Regularization}

L1 regularization promotes sparse structure in the $A$ matrix:
\begin{equation}
\mathcal{L}_{\text{sparse}} = \lambda_s \|A\|_1 = \lambda_s \sum_{i,j}|A_{ij}|
\end{equation}

\begin{lemma}[Structured Representation via Block Sparsity]
\label{lem:block_sparse}
Under Definition~\ref{def:augmented}, the generator matrix partitions as:
\begin{equation}
A = \begin{pmatrix} A_{\text{lib-lib}} & A_{\text{lib-mlp}} \\ A_{\text{mlp-lib}} & A_{\text{mlp-mlp}} \end{pmatrix}
\end{equation}
If $W_{\text{lib}} = I$ (identity projection) and $A_{\text{lib-mlp}} = 0$, then the dynamics of the $i$-th monomial observable satisfy:
\begin{equation}
\frac{d}{dt}\psi_i(u) = \sum_{j} [A_{\text{lib-lib}}]_{ij} \psi_j(u)
\label{eq:symbolic_dynamics}
\end{equation}
where each non-zero $[A_{\text{lib-lib}}]_{ij}$ indicates that monomial $\psi_j$ contributes to the time evolution of $\psi_i$. (Proof in Appendix~\ref{app:proofs}.)
\end{lemma}

Non-zero entries in the library portion of $A$ directly indicate active polynomial terms, embodying the parsimony principle \citep{ma2022principles}.

\begin{remark}[Connection to Koopman Eigenfunction Theory]
\label{rem:minimal_set}
The observed sparsity in learned $A$ matrices aligns with the minimal set theorem of \citet{cohen2025koopman}: an $N$-dimensional dynamical system admits exactly $N$ functionally independent Koopman eigenfunctions. The empirical observation that active dimensions $\ll$ latent dimension (Table~\ref{tab:sparsity}) suggests the learned representation captures a near-minimal generating set, explaining why L1 regularization succeeds.
\end{remark}

\subsection{Training Objective}

The full training objective combines PDE residual, Koopman consistency, and sparsity:
\begin{equation}
\mathcal{L}_{\text{total}} = \underbrace{\mathbb{E}\left[|\mathcal{N}[u] - f|^2\right]}_{\text{PDE solution}} + \lambda_k \mathcal{L}_{\text{koopman}} + \lambda_s\|A\|_1 + \lambda_{ic}\mathcal{L}_{IC} + \lambda_{bc}\mathcal{L}_{BC}
\end{equation}

The Koopman loss $\mathcal{L}_{\text{koopman}}$ admits multiple formulations corresponding to different integrators:
\begin{itemize}
\item \textbf{Euler}: $\mathcal{L}_{\text{koopman}}^{\text{Euler}} = \mathbb{E}[\|z_{n+1} - (z_n + Az_n\Delta t)\|^2]$
\item \textbf{RK4}: $\mathcal{L}_{\text{koopman}}^{\text{RK4}} = \mathbb{E}[\|z_{n+1} - \text{RK4}(z_n, A, \Delta t)\|^2]$
\item \textbf{EXPM}: $\mathcal{L}_{\text{koopman}}^{\text{EXPM}} = \mathbb{E}[\|z_{n+1} - e^{A\Delta t}z_n\|^2]$ (Equation~\ref{eq:expm_propagation})
\end{itemize}
All formulations enforce $\dot{z} = Az$ but differ in numerical stability for stiff systems. Training pairs $(z_n, z_{n+1})$ are sampled at consecutive timesteps $t_n, t_{n+1} = t_n + \Delta t$ from collocation points, where $z_n = g(u_\theta(x, t_n))$ is computed by evaluating the PINN at each timestep. This discrete formulation avoids explicit computation of $\dot{z}$ via autograd, instead comparing propagated states directly. The timestep $\Delta t = 0.01$ is fixed across all systems (see Appendix~\ref{app:hyperparameters}).

Loss weights follow standard multi-task physics-informed learning practice \citep{wang2021understanding}; details in Appendix~\ref{app:hyperparameters}.

The proposed framework is referred to as \textbf{PIKE} (Physics-Informed Koopman-Enhanced) and \textbf{SPIKE} (Sparse PIKE). Both yield sparse generator matrices $A$: PIKE learns sparse structure implicitly through the Koopman constraint, while SPIKE explicitly enforces sparsity via L1 regularization on $A$. Collectively, PIKE/SPIKE denotes the Koopman-regularized PINN approach.

\section{Experiments}
\label{sec:experiments}

\subsection{Experimental Setup}

Experiments span 14 dynamical systems: 9 one-dimensional PDEs (Heat, Advection, Burgers, Allen-Cahn, KdV, Reaction-Diffusion, Cahn-Hilliard, Kuramoto-Sivashinsky, Schr\"odinger), 3 two-dimensional PDEs (2D Wave, 2D Burgers, Navier-Stokes), and 2 ODEs (Lorenz, SEIR). This selection covers parabolic, hyperbolic, dispersive, chaotic, and stiff dynamics. Full system descriptions are provided in Appendix~\ref{app:systems}.

All models use identical architectures (4-layer MLP, 128 units, tanh activation) with embedding dimension 64. Training proceeds for 5000 steps; loss weights $\lambda_{\text{koopman}} = 0.1$, $\lambda_{\text{sparse}} = 0.01$ are fixed across all systems (Appendix~\ref{app:hyperparameters}). To ensure fair comparison, the PINN baseline uses the same architecture, optimizer (Adam, lr=$10^{-3}$), and training budget as PIKE/SPIKE; the only difference is the addition of Koopman regularization. Extended training (10,000+ steps) was tested for PINN on stiff systems (Cahn-Hilliard, 2D Wave) but did not resolve the failure modes, confirming these are fundamental limitations rather than undertrained baselines.

\subsection{Main Results}

Table~\ref{tab:summary_results} summarizes key performance improvements across all systems. PIKE/SPIKE achieves consistent gains over standard PINNs, with the most substantial improvements on stiff and chaotic systems.

\begin{table}[!htb]
\centering
\caption{Summary of Key Improvements over PINN Baseline (grouped by metric type)}
\label{tab:summary_results}
\begin{tabular}{llcc}
\toprule
System & Metric & Improvement & Best Method \\
\midrule
\multicolumn{4}{l}{\textit{In-Domain MSE}} \\
2D Wave & In-domain & \textbf{$8 \times 10^7\times$} & PIKE-Euler \\
Cahn-Hilliard & In-domain & \textbf{$10^6\times$} & PIKE-Euler \\
\midrule
\multicolumn{4}{l}{\textit{OOD-Space MSE}$^*$} \\
2D Burgers$^\dagger$ & $xy \in [1,2]$ & \textbf{38$\times$} & PIKE-Euler \\
Advection & $x \in [3,5]$ & \textbf{29$\times$} & SPIKE-EXPM \\
Allen-Cahn$^\dagger$ & $x \in [3,5]$ & \textbf{7.5$\times$} & PIKE-Euler \\
Navier-Stokes & $xy \in [1,2]$ & \textbf{32$\times$} & PIKE-Euler \\
Kuramoto-Sivashinsky & $x \in [3,5]$ & \textbf{2.1$\times$} & SPIKE-EXPM \\
\midrule
\multicolumn{4}{l}{\textit{OOD-Time MSE}} \\
Schr\"odinger & $t \in [3,5]$ & \textbf{24$\times$} & SPIKE-EXPM \\
KdV & $t \in [3,5]$ & \textbf{6.3$\times$} & PIKE-EXPM \\
Burgers & $t \in [3,5]$ & \textbf{2.4$\times$} & PIKE-Euler \\
Kuramoto-Sivashinsky & $t \in [3,5]$ & \textbf{2.8$\times$} & PIKE-RK4 \\
\midrule
\multicolumn{4}{l}{\textit{Chaotic Systems (Valid Prediction Time)}} \\
Lorenz & Valid time & \textbf{184$\times$} & PIKE-Euler \\
\bottomrule
\end{tabular}
\vspace{0.3em}

\footnotesize{$^*$OOD-Space tests extrapolation beyond training domain.\\ $^\dagger$Bounded-domain problems where spatial extrapolation demonstrates model capability rather than physical prediction.}
\end{table}

Figure~\ref{fig:main_results} compares solution quality for representative PDEs, showing that PIKE/SPIKE maintains accuracy in OOD regions where PINN exhibits oscillations and phase drift.
\begin{figure}[!htb]
\centering
\includegraphics[width=0.85\textwidth]{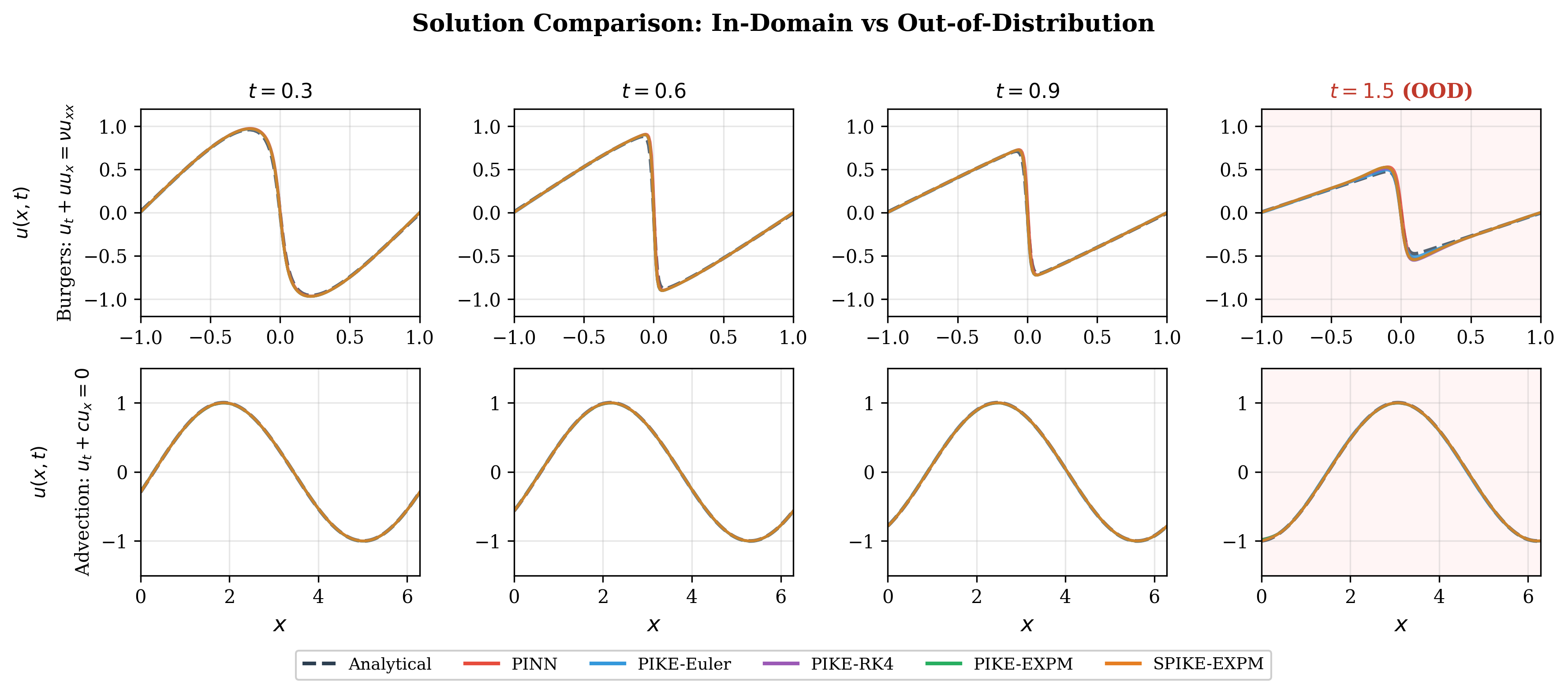}
\caption{Solution comparison for Burgers and Advection equations. Columns show temporal evolution from in-domain ($t \leq 1$) to OOD extrapolation ($t=1.5$). PINN exhibits oscillations near the Burgers shock and phase drift in Advection; PIKE/SPIKE variants maintain accuracy through Koopman regularization. Quantitative MSE values in Appendix~\ref{app:detailed_results} (Tables~\ref{tab:integrator_ablation}--\ref{tab:integrator_ood_space}).}
\label{fig:main_results}
\end{figure}

\textbf{Challenging dynamics.} For chaotic systems, PIKE-Euler achieves 184$\times$ improvement in valid prediction time for Lorenz (12.91s vs 0.07s, corresponding to 11.7 vs 0.06 Lyapunov times). For 2D Navier-Stokes channel flow (an open-domain system), PIKE-Euler yields 32$\times$ lower OOD-Space MSE, representing physically meaningful downstream prediction. The improvement magnitude correlates with system stiffness and nonlinearity: fourth-order PDEs (Cahn-Hilliard: $10^6\times$) and chaotic systems benefit most from Koopman regularization, while well-posed parabolic systems (Heat) show moderate but consistent gains. For bounded-domain systems (2D Burgers, Allen-Cahn), OOD-Space improvements reflect extrapolation regularity---Koopman regularization prevents catastrophic failure and learns smoother representations, even though spatial extrapolation is not physically meaningful for these systems.

\textbf{Integrator comparison.} PIKE-Euler achieves lowest error on non-stiff systems; SPIKE-EXPM excels on stiff PDEs where EXPM provides unconditional stability (Euler: $\max\text{Re}(\lambda) = 0.21$; EXPM: $0.00$). Computational overhead: $<$5\% for Euler/RK4, $\sim$25\% for EXPM (Appendix~\ref{app:computational}).

\textbf{Sparsity and interpretability.} L1 regularization yields up to 5.7$\times$ reduction in non-zero entries (Schr\"odinger: $17 \to 3$). Latent MLP features show high correlations with derivatives (0.99 with $u_{xx}$ for Heat), enabling post-hoc coefficient recovery with $<1\%$ error (Appendix~\ref{app:detailed_results},~\ref{app:interpretability}).

\section{Discussion}
\label{sec:discussion}

\textbf{Key insights.} Three main findings emerge: (1) Koopman regularization acts as a dynamics-aware constraint that prevents PINN overfitting, yielding 2--184$\times$ improvements in temporal extrapolation; (2) the continuous-time formulation with matrix exponential integration provides unconditional stability for stiff PDEs where standard PINNs fail; (3) sparse generator matrices ($\leq$99.9\% zeros) enable interpretable dynamics recovery while maintaining accuracy. Improvements are most pronounced for stiff PDEs, temporal extrapolation, and chaotic systems (Lorenz: 184$\times$ longer prediction). These results align with Proposition~\ref{prop:ood_bound}: smaller $\epsilon_K$ (tight Koopman fit) and bounded $\rho_0$ (stable eigenvalues) directly reduce extrapolation error.

\textbf{Temporal vs.\ spatial extrapolation.} Temporal OOD is physically meaningful for all systems. Spatial OOD tests downstream prediction for open-domain problems (channel flow, Advection) and extrapolation regularity for bounded-domain problems---demonstrating that Koopman regularization prevents catastrophic failure outside training regions (details in Appendix~\ref{app:detailed_results}).

Unlike prior physics-informed Koopman methods \citep{liu2022physicskoopman} that require explicit dynamics $f(x)$, SPIKE operates directly from PDE residuals. The library-latent decomposition enables both structured interpretation and correlation-based analysis. The current experiments use polynomial-only libraries to evaluate implicit derivative learning; the library can be extended to include derivative terms ($u_x$, $u_{xx}$, $u \cdot u_x$) for explicit representation of convective and diffusive dynamics.

\textbf{On improvement magnitudes.} The extreme improvements for certain systems (Cahn-Hilliard, 2D Wave) reflect genuine PINN failure modes: fourth-order stiffness and phase coherence requirements that unconstrained networks cannot maintain. Koopman regularization addresses these fundamental limitations through structural constraints on learned dynamics.

\textbf{Integrator selection.} The choice between Euler, RK4, and EXPM involves accuracy-stability tradeoffs. Euler is fastest but conditionally stable; RK4 provides fourth-order accuracy; EXPM is unconditionally stable, critical for stiff PDEs ($|\lambda_{\max}| > 10^4$) at $\sim$25\% overhead. The OOD bound (Proposition~\ref{prop:ood_bound}) depends on both $\epsilon_K$ and $\rho_0$: for moderately stiff systems, Euler's tighter Koopman fit can outweigh mild instability over short extrapolation windows, while EXPM becomes essential for highly stiff systems.

\textbf{Navier-Stokes case study.} The 2D Navier-Stokes results illustrate integrator selection: PIKE-Euler achieves 82$\times$ lower in-domain error for channel flow and 2.3$\times$ for lid-driven cavity, with corresponding OOD improvements (32$\times$ spatial, 15\% temporal). Where PIKE does not outperform PINN, both achieve comparable performance within the same order of magnitude.

\textbf{Inverse problem capability.} Beyond forward simulation, PIKE/SPIKE enables post-hoc coefficient recovery: after training, PDE coefficients can be estimated via least-squares regression on PINN derivatives without retraining. Across 9 PDEs/ODEs, coefficient recovery achieves $<1\%$ relative error (Heat $u_{xx}$: 0.33\%, Burgers $u \cdot u_x$: 0.01\%, KdV $u_{xxx}$: 0.21\%, Lorenz $\sigma, \rho, \beta$: $<0.3\%$). This leverages smooth autograd derivatives from $u_\theta$, avoiding noise amplification from finite differences. Details in Appendix~\ref{app:interpretability}.

\textbf{Limitations.} The polynomial library component of $g(u)$ assumes that relevant dynamics can be captured by low-degree polynomials in $u$. For PDEs with strongly nonlinear or non-polynomial terms (e.g., $\sin(u)$, $e^u$), the library may require extension. However, the latent MLP component partially compensates by learning implicit correlations with derivative structure (up to 0.99 with $u_{xx}$), though these remain non-symbolic. The current framework also assumes smooth solutions; discontinuities (shocks in hyperbolic PDEs) may require specialized treatment. Finally, spatial OOD improvements for bounded-domain problems should be interpreted as extrapolation regularity rather than physically meaningful prediction.

\section{Conclusion}

This work presented PIKE/SPIKE, combining PINNs with continuous-time Koopman operators for parsimonious dynamics representations with improved generalization. The augmented embedding decomposes observables into interpretable polynomial and flexible MLP components, avoiding identity collapse while enabling L1 sparsity (up to 5.7$\times$ reduction). Experiments on 14 systems demonstrate consistent improvement in temporal extrapolation, with Koopman regularization preventing catastrophic failure outside training regions. For open-domain systems, this yields physically meaningful spatial generalization; for bounded domains, it ensures smooth extrapolation behavior. The latent MLP's correlation with derivative structure (up to 0.99 with $u_{xx}$) suggests derivative-augmented libraries as a promising extension.

\bibliography{reference}

\appendix
\section{System Descriptions}
\label{app:systems}

The benchmark suite spans 14 dynamical systems selected to cover diverse mathematical structures and physical phenomena. The selection includes: (1) \textit{parabolic} PDEs (Heat, Reaction-Diffusion) testing diffusion-dominated dynamics; (2) \textit{hyperbolic} PDEs (Advection, Wave) testing transport and propagation; (3) \textit{dispersive} PDEs (KdV, Schr\"odinger) testing wave dispersion and soliton dynamics; (4) \textit{stiff} PDEs (Cahn-Hilliard, Kuramoto-Sivashinsky) testing fourth-order operators with extreme eigenvalue spreads; (5) \textit{nonlinear} PDEs (Burgers, Allen-Cahn, Navier-Stokes) testing shock formation and bistable dynamics; and (6) \textit{chaotic} systems (Lorenz) testing sensitive dependence on initial conditions. This comprehensive coverage ensures that observed improvements are not artifacts of specific equation types.

The following dynamical systems are evaluated:

\textbf{1D PDEs:}
\begin{itemize}
\item \textbf{Heat}: $u_t = \alpha u_{xx}$ (parabolic diffusion)
\item \textbf{Advection}: $u_t + c u_x = 0$ (hyperbolic transport)
\item \textbf{Burgers}: $u_t + u u_x = \nu u_{xx}$ (nonlinear convection-diffusion)
\item \textbf{Allen-Cahn}: $u_t = \epsilon u_{xx} + u - u^3$ (bistable reaction-diffusion)
\item \textbf{KdV}: $u_t + u u_x + u_{xxx} = 0$ (dispersive solitons)
\item \textbf{Reaction-Diffusion}: $u_t = D u_{xx} + R(u)$ (pattern formation)
\item \textbf{Cahn-Hilliard}: $u_t = -\epsilon^2 u_{xxxx} + (u^3 - u)_{xx}$ (fourth-order phase separation)
\item \textbf{Kuramoto-Sivashinsky}: $u_t + u u_x + u_{xx} + u_{xxxx} = 0$ (spatiotemporal chaos)
\item \textbf{Schr\"odinger}: $i u_t + u_{xx} + |u|^2 u = 0$ (complex dispersive)
\end{itemize}

\textbf{2D PDEs:}
\begin{itemize}
\item \textbf{2D Wave}: $u_{tt} = c^2 (u_{xx} + u_{yy})$ (hyperbolic propagation)
\item \textbf{2D Burgers}: $u_t + u u_x + v u_y = \nu (u_{xx} + u_{yy})$
\item \textbf{Navier-Stokes}: $\mathbf{u}_t + (\mathbf{u} \cdot \nabla)\mathbf{u} = -\nabla p + \nu \nabla^2 \mathbf{u}$ (incompressible flow)
\end{itemize}

\textbf{ODEs:}
\begin{itemize}
\item \textbf{Lorenz}: $\dot{x} = \sigma(y-x), \dot{y} = x(\rho-z)-y, \dot{z} = xy - \beta z$ (chaotic attractor, $\sigma=10$, $\rho=28$, $\beta=8/3$)
\item \textbf{SEIR}: $\dot{S} = -\beta SI/N, \dot{E} = \beta SI/N - \sigma E, \dot{I} = \sigma E - \gamma I, \dot{R} = \gamma I$
\end{itemize}

\subsection{Reproduction Details}
\label{app:reproduction}

We provide complete specifications for experimental reproducibility. Let $\Omega_{\text{train}} \subset \mathbb{R}^{d} \times \mathbb{R}^{+}$ denote the training domain and $\Omega_{\text{OOD}}$ the out-of-distribution evaluation domain.

\subsubsection{Domain Specifications and Boundary Conditions}

\textbf{1D PDEs.} For all one-dimensional systems, the training domain is $\Omega_{\text{train}} = [0,1] \times [0,1]$ (spatial $\times$ temporal). Table~\ref{tab:reproduction_1d} specifies initial conditions $u(x,0) = u_0(x)$, boundary conditions, and OOD evaluation domains.

\begin{table}[H]
\centering
\caption{1D PDE Specifications: Initial/Boundary Conditions and Evaluation Domains}
\label{tab:reproduction_1d}
\footnotesize
\begin{tabular}{llll}
\toprule
\textbf{System} & \textbf{Initial Condition} $u_0(x)$ & \textbf{Boundary Conditions} & \textbf{OOD Domain} \\
\midrule
Heat & $\sin(\pi x)$ & $u(0,t) = u(1,t) = 0$ & $t \in [1,3]$ \\
Advection & $\sin(2\pi x)$ & Periodic: $u(0,t) = u(1,t)$ & $x \in [1,3]$ \\
Burgers & $-\sin(\pi x)$ & $u(0,t) = u(1,t) = 0$ & $t \in [1,3]$ \\
Allen-Cahn & $x^2\cos(\pi x)$ & $\partial_x u(0,t) = \partial_x u(1,t) = 0$ & $t \in [1,3]$ \\
KdV & $2\,\text{sech}^2(x - 0.5)$ & Periodic: $u(0,t) = u(1,t)$ & $t \in [1,3]$ \\
Reaction-Diff & $\exp(-50(x-0.5)^2)$ & $\partial_x u(0,t) = \partial_x u(1,t) = 0$ & $t \in [1,3]$ \\
Cahn-Hilliard & $0.1\mathcal{N}(0,1)$ perturbation & Periodic: $u, \partial_x u$ periodic & $t \in [1,3]$ \\
KS & $\cos(x)(1+\sin(x))$ & Periodic: $u(0,t) = u(2\pi,t)$ & $t \in [1,3]$ \\
Schr\"odinger & $\text{sech}(x-0.5)e^{2ix}$ & Periodic: $u(0,t) = u(1,t)$ & $t \in [1,3]$ \\
\bottomrule
\end{tabular}
\end{table}

\textbf{2D PDEs.} The training domain is $\Omega_{\text{train}} = [0,1]^2 \times [0,1]$. For Navier-Stokes, we evaluate two configurations: (i) channel flow with parabolic inlet $u(0,y,t) = 4y(1-y)$, $v=0$, and outflow $\partial_x u = 0$; (ii) lid-driven cavity with $\mathbf{u}(x,1,t) = (1,0)$ and no-slip elsewhere. Both use $\text{Re} = 100$.

\begin{table}[H]
\centering
\caption{2D PDE Specifications}
\label{tab:reproduction_2d}
\footnotesize
\begin{tabular}{llll}
\toprule
\textbf{System} & \textbf{Initial Condition} & \textbf{Boundary Conditions} & \textbf{OOD Domain} \\
\midrule
2D Wave & $\exp(-50((x-0.5)^2+(y-0.5)^2))$ & $u|_{\partial\Omega} = 0$ (Dirichlet) & $x,y \in [1,2]$ \\
2D Burgers & $\exp(-20((x-0.5)^2+(y-0.5)^2))$ & $u|_{\partial\Omega} = 0$ (Dirichlet) & $x,y \in [1,2]$ \\
Navier-Stokes & Quiescent: $\mathbf{u}(\mathbf{x},0) = \mathbf{0}$ & No-slip / Inlet-Outlet & $x,y \in [1,2]$ \\
\bottomrule
\end{tabular}
\end{table}

\textbf{ODEs.} For Lorenz, $\mathbf{x}_0 = (1, 1, 1)^T$ with parameters $(\sigma, \rho, \beta) = (10, 28, 8/3)$. For SEIR, $(S_0, E_0, I_0, R_0) = (0.99, 0.01, 0, 0)$ with $(\beta, \sigma, \gamma) = (0.4, 0.2, 0.1)$ and $N=1$. Training domain: $t \in [0,1]$; OOD: $t \in [1,15]$ (Lorenz), $t \in [1,3]$ (SEIR).

\subsubsection{Collocation Point Sampling}

Training points are sampled via Latin Hypercube Sampling (LHS) to ensure uniform coverage. Let $N_{\text{col}}$, $N_{\text{bc}}$, $N_{\text{ic}}$ denote the number of interior collocation, boundary, and initial condition points, respectively:
\begin{itemize}
\item \textbf{1D PDEs}: $N_{\text{col}} = 10{,}000$, $N_{\text{bc}} = 200$ (100 per boundary), $N_{\text{ic}} = 100$
\item \textbf{2D PDEs}: $N_{\text{col}} = 20{,}000$, $N_{\text{bc}} = 500$ (125 per face), $N_{\text{ic}} = 500$
\item \textbf{ODEs}: $N_{\text{col}} = 5{,}000$, $N_{\text{ic}} = 1$ (single initial condition)
\end{itemize}

\subsubsection{Evaluation Protocol}

Physics residual MSE is computed on a held-out uniform grid $\mathcal{G}_{\text{test}}$ disjoint from training points:
\begin{equation}
\text{MSE}_{\text{physics}} = \frac{1}{|\mathcal{G}_{\text{test}}|} \sum_{(\mathbf{x}, t) \in \mathcal{G}_{\text{test}}} \|\mathcal{N}[u_\theta](\mathbf{x}, t)\|^2
\end{equation}
where $\mathcal{N}[\cdot]$ is the PDE residual operator. Grid resolutions: $100 \times 100$ (1D PDEs), $50 \times 50 \times 50$ (2D PDEs), $1{,}000$ temporal points (ODEs).

\subsubsection{Ground Truth Generation}

Reference solutions are computed as follows:
\begin{itemize}
\item \textbf{Analytical}: Heat (separation of variables), Advection (method of characteristics), KdV (inverse scattering), 2D Wave (d'Alembert)
\item \textbf{Spectral methods}: Allen-Cahn, Cahn-Hilliard (Fourier-Galerkin with $N=256$ modes, implicit-explicit time stepping, $\Delta t = 10^{-4}$)
\item \textbf{ETDRK4}: Kuramoto-Sivashinsky (exponential time differencing fourth-order Runge-Kutta, $N=512$ modes)
\item \textbf{Split-step Fourier}: Schr\"odinger ($N=256$ modes, $\Delta t = 10^{-5}$)
\item \textbf{Finite differences}: Burgers (Cole-Hopf transformation for analytical verification), 2D Burgers (central differences, RK4, $\Delta x = \Delta y = 0.01$)
\item \textbf{Finite volume}: Navier-Stokes (staggered grid, SIMPLE algorithm, $\Delta x = \Delta y = 0.02$, Re $= 100$)
\item \textbf{Adaptive integration}: Lorenz, SEIR (SciPy \texttt{solve\_ivp} with RK45, absolute and relative tolerance $10^{-10}$)
\end{itemize}

All numerical reference solutions are validated against analytical solutions (where available) or refined meshes to ensure discretization error is negligible ($<10^{-6}$) relative to reported MSE values.

\FloatBarrier
\section{Method Comparison}
\label{app:method_comparison}

Table~\ref{tab:related_comparison} compares SPIKE with related methods across key dimensions.

\begin{table}[H]
\centering
\caption{Comparison of PIKE/SPIKE with Related Methods}
\label{tab:related_comparison}
\begin{tabular}{lccccc}
\toprule
Method & Requires $f(x)$ & OOD Target & PDE Solving & Koopman \\
\midrule
PINN \citep{raissi2019pinn} & No & No & Yes & No \\
Variance Loss \citep{hanna2025variance} & No & No & Yes & No \\
SINDy \citep{brunton2016sindy} & Yes (samples) & No & No & No \\
Deep Koopman \citep{lusch2018deep} & Yes (trajectories) & No & No & Yes \\
PIKN \citep{liu2022physicskoopman} & Yes (explicit) & No & No & Yes \\
Koopman Reg. \citep{cohen2025koopman} & Yes (samples) & Yes & No & Yes \\
\textbf{PIKE/SPIKE} & \textbf{No} & \textbf{Yes} & \textbf{Yes} & \textbf{Yes} \\
\bottomrule
\end{tabular}
\end{table}

\begin{table}[H]
\centering
\caption{Detailed Methodological Comparison. Inverse column indicates coefficient recovery capability: ``Training'' requires unknown coefficients as trainable parameters during optimization; ``Post-hoc'' enables coefficient recovery after training on forward problems without retraining.}
\label{tab:method_details}
\begin{tabular}{lcccc}
\toprule
Method & Derivatives & Physics Loss & Sparsity & Inverse \\
\midrule
PINN \citep{raissi2019pinn} & Autograd & PDE residual & No & Training \\
SINDy \citep{brunton2016sindy} & Finite diff & None & LASSO & Post-hoc \\
Deep Koopman \citep{lusch2018deep} & None & None & No & No \\
PIKN \citep{liu2022physicskoopman} & Autograd & Lie operator & No & No \\
GN-SINDy \citep{tang2024gnsindy} & NN-based & None & SINDy & Yes \\
\textbf{SPIKE} & \textbf{Autograd} & \textbf{PDE residual} & \textbf{L1 on $A$} & \textbf{Post-hoc} \\
\bottomrule
\end{tabular}
\end{table}

The comparison tables highlight three key distinctions. First, PIKE/SPIKE is the only method that combines Koopman structure with direct PDE solving capability; existing Koopman methods operate on trajectory data or require explicit dynamics $f(x)$, while PIKE/SPIKE works directly from PDE residuals via autograd. Second, PIKE/SPIKE explicitly targets OOD generalization, whereas most methods focus on training accuracy or system identification. Third, the continuous-time formulation $dz/dt = Az$ avoids the identity collapse problem inherent in discrete Koopman operators, enabling meaningful sparsity analysis regardless of timestep size.

\FloatBarrier
\section{Hyperparameter Settings}
\label{app:hyperparameters}

Table~\ref{tab:hyperparameters} summarizes the hyperparameter settings used across all PDE systems. A unified configuration was employed to ensure fair comparison: all models use identical network architectures (4 hidden layers, 128 units each) with the same embedding dimension and polynomial degree. The loss weights $\lambda_{\text{koopman}} = 0.1$ and $\lambda_{\text{sparse}} = 0.01$ were selected to balance gradient magnitudes across loss terms, following standard practice in physics-informed learning \citep{raissi2019pinn, wang2021understanding}. During training, the observed magnitudes were $\mathcal{L}_{\text{physics}} \sim 10^{-4}$ and $\mathcal{L}_{\text{koopman}} \sim 10^{-3}$, so $\lambda_{\text{koopman}} = 0.1$ yields comparable gradient contributions. These weights were held fixed across all systems to avoid per-system tuning and demonstrate robustness of the approach.

\begin{table}[H]
\centering
\caption{Hyperparameter Settings Across PDE Systems}
\label{tab:hyperparameters}
\begin{tabular}{lc|lc}
\toprule
\textbf{Architecture} & \textbf{Value} & \textbf{Training} & \textbf{Value} \\
\midrule
Hidden layers & 4 & $\lambda_{\text{koopman}}$ & 0.1 \\
Hidden units & 128 & $\lambda_{\text{sparse}}$ & 0.01 \\
Embedding dim & 64 & Koopman $\Delta t$ & 0.01 \\
Poly degree & 2 & Training steps & 5000 \\
Activation & tanh & Optimizer & Adam \\
\bottomrule
\end{tabular}
\end{table}

\textbf{Note}: Identical hyperparameters are used across all 14 systems (9 1D PDEs, 3 2D PDEs, 2 ODEs) to ensure fair comparison and demonstrate robustness. No per-system tuning was performed.

The polynomial degree $d=2$ provides a library of $\binom{n+2}{2}$ terms (6 terms for scalar PDEs: $[1, u, u^2]$ plus MLP features). Higher degrees were tested but yielded diminishing returns while increasing computational cost. The embedding dimension of 64 balances expressivity with the sparsity objective; larger dimensions provide more capacity but require stronger regularization to achieve comparable sparsity levels.

\FloatBarrier
\section{Detailed Quantitative Results}
\label{app:detailed_results}

This section provides the full quantitative analysis supporting the main experimental claims.

\subsection{Integrator Ablation}

Table~\ref{tab:integrator_ablation} compares the in-domain physics residual MSE across different Koopman integrators. PIKE-Euler achieves best performance on non-stiff systems, while EXPM variants excel on stiff PDEs (Cahn-Hilliard, Kuramoto-Sivashinsky).

\begin{table}[H]
\centering
\caption{Integrator Ablation: In-Domain Physics Residual MSE}
\label{tab:integrator_ablation}
\begin{tabular}{lccccc}
\toprule
System & PINN & PIKE-Euler & PIKE-RK4 & PIKE-EXPM & SPIKE-EXPM \\
\midrule
Heat & 1.01e-05 & \textbf{6.95e-06} & 1.20e-05 & 6.02e-06 & 8.67e-06 \\
Advection & \textbf{5.04e-07} & 8.89e-05 & 4.45e-07 & 4.88e-07 & 3.90e-07 \\
Burgers & 9.68e-05 & 4.63e-05 & 3.44e-05 & \textbf{3.23e-05} & \textbf{3.23e-05} \\
Allen-Cahn & \textbf{1.19e-05} & 1.56e-05 & 1.44e-05 & 2.34e-05 & 7.08e-05 \\
KdV & \textbf{5.23e-02} & 4.75e-02 & 5.24e-02 & 5.06e-02 & 5.14e-02 \\
Reaction-Diffusion & \textbf{5.02e-05} & 2.67e-04 & 4.80e-05 & 4.63e-05 & 1.51e-04 \\
Cahn-Hilliard & 3.53e-01 & \textbf{2.61e-07} & 3.50e-01 & 3.46e-01 & 1.22e-01 \\
Schr\"odinger & 2.50e+01 & 2.53e+01 & 2.49e+01 & 2.47e+01 & \textbf{1.20e+01} \\
Kuramoto-Sivashinsky & \textbf{1.83e+01} & 8.77e+01 & 6.75e+02 & 2.93e+02 & 1.91e+01 \\
\bottomrule
\end{tabular}
\end{table}

\begin{table}[H]
\centering
\caption{Integrator Ablation: OOD-Time MSE ($t \in [3,5]$)}
\label{tab:integrator_ood_time}
\begin{tabular}{lccccc}
\toprule
System & PINN & PIKE-Euler & PIKE-RK4 & PIKE-EXPM & SPIKE-EXPM \\
\midrule
Heat & 7.52e-04 & 2.17e-03 & \textbf{7.00e-04} & 7.97e-04 & 8.03e-04 \\
Advection & \textbf{4.05e-10} & 2.96e-09 & 6.02e-10 & 4.48e-10 & 1.06e-09 \\
Burgers & 2.77e-02 & \textbf{1.15e-02} & 2.97e-02 & 2.91e-02 & 2.91e-02 \\
KdV & 8.89e-03 & 2.49e-02 & 9.99e-03 & \textbf{1.42e-03} & 1.81e-03 \\
Schr\"odinger & 3.09e+01 & 2.86e+01 & 3.69e+01 & 4.19e+01 & \textbf{1.29e+00} \\
\bottomrule
\end{tabular}
\end{table}

\begin{table}[H]
\centering
\caption{Integrator Ablation: OOD-Space MSE ($x \in [3,5]$)}
\label{tab:integrator_ood_space}
\begin{tabular}{lccccc}
\toprule
System & PINN & PIKE-Euler & PIKE-RK4 & PIKE-EXPM & SPIKE-EXPM \\
\midrule
Heat$^\ddagger$ & 1.52e-05 & \textbf{7.57e-06} & 1.53e-05 & 1.50e-05 & 1.46e-05 \\
Advection & 1.51e-06 & 5.95e-05 & 9.21e-07 & 8.66e-07 & \textbf{5.12e-08} \\
Burgers$^\ddagger$ & \textbf{6.56e-03} & 1.67e-02 & 7.70e-03 & 6.97e-03 & 6.97e-03 \\
Allen-Cahn$^\ddagger$ & 1.35e-01 & \textbf{1.81e-02} & 1.55e-01 & 6.36e-02 & 5.18e-02 \\
Cahn-Hilliard$^\ddagger$ & 1.31e+01 & \textbf{3.36e-06} & 1.28e+01 & 1.05e+01 & 5.24e-01 \\
\bottomrule
\end{tabular}
\vspace{0.5em}
\footnotesize{$^\ddagger$Bounded-domain problems where spatial extrapolation demonstrates model capability rather than physical prediction.}
\end{table}

\begin{table}[H]
\centering
\caption{Integrator Ablation: 2D PDEs In-Domain MSE}
\label{tab:integrator_2d}
\begin{tabular}{lccccc}
\toprule
System & PINN & PIKE-Euler & PIKE-RK4 & PIKE-EXPM & SPIKE-EXPM \\
\midrule
Wave 2D & 1.72e+04 & \textbf{2.06e-04} & 2.17e+05 & 7.69e+04 & 5.11e-02 \\
Burgers 2D & 9.66e-03 & \textbf{7.12e-03} & 4.92e-02 & 3.23e-02 & 9.34e-03 \\
Navier-Stokes & 1.09e+01 & \textbf{1.33e-01} & 4.31e+00 & 2.25e+01 & 2.47e+01 \\
NS Lid-Driven$^\dagger$ & 3.74e-01 & \textbf{1.62e-01} & 3.33e-01 & 1.66e-01 & 5.25e-01 \\
\bottomrule
\end{tabular}
\vspace{0.5em}
\footnotesize{$^\dagger$Lid-driven cavity: PIKE-Euler achieves 2.3$\times$ lower in-domain MSE than PINN.}
\end{table}

\begin{table}[H]
\centering
\caption{Integrator Ablation: 2D PDEs OOD-Space MSE ($xy \in [1,2]$)}
\label{tab:integrator_2d_ood_space}
\begin{tabular}{lccccc}
\toprule
System & PINN & PIKE-Euler & PIKE-RK4 & PIKE-EXPM & SPIKE-EXPM \\
\midrule
Wave 2D$^\ddagger$ & 1.14e+00 & \textbf{1.32e-04} & 1.35e+00 & 1.38e-01 & 5.20e-03 \\
Burgers 2D$^\ddagger$ & 6.19e-03 & \textbf{1.62e-04} & 2.42e-04 & 2.79e-04 & 2.99e-03 \\
Navier-Stokes & 3.28e+00 & \textbf{1.04e-01} & 2.42e+00 & 1.18e+01 & 7.76e+00 \\
NS Lid-Driven$^\ddagger$ & 6.36e+00 & 7.88e+00 & \textbf{6.43e+00} & 1.24e+01 & 1.67e+01 \\
\bottomrule
\end{tabular}
\vspace{0.5em}
\footnotesize{$^\ddagger$OOD-Space is not physically meaningful for bounded-domain problems (fixed boundaries); results demonstrate extrapolation capability rather than physical prediction.}
\end{table}

\begin{table}[H]
\centering
\caption{Integrator Ablation: 2D PDEs OOD-Time MSE ($t \in [1,3]$)}
\label{tab:integrator_2d_ood_time}
\begin{tabular}{lccccc}
\toprule
System & PINN & PIKE-Euler & PIKE-RK4 & PIKE-EXPM & SPIKE-EXPM \\
\midrule
Wave 2D & 3.19e+00 & \textbf{2.06e-04} & 6.23e-04 & 1.81e-02 & 1.16e-02 \\
Burgers 2D & 7.87e-02 & 9.29e-02 & 6.80e-02 & 7.88e-02 & \textbf{6.12e-02} \\
Navier-Stokes & \textbf{2.24e-02} & 2.78e-02 & 2.94e-02 & 2.67e-01 & 8.93e-01 \\
NS Lid-Driven & 6.32e-02 & 1.17e-01 & \textbf{5.40e-02} & 1.18e-01 & 1.19e+00 \\
\bottomrule
\end{tabular}
\end{table}

\begin{table}[H]
\centering
\caption{Integrator Ablation: ODE Systems In-Domain MSE}
\label{tab:integrator_ode}
\begin{tabular}{lccccc}
\toprule
System & PINN & PIKE-Euler & PIKE-RK4 & PIKE-EXPM & SPIKE-EXPM \\
\midrule
Lorenz & 5.48e+04 & 4.79e+04 & 5.03e+04 & \textbf{4.35e+04} & \textbf{4.35e+04} \\
SEIR & 4.70e-02 & \textbf{4.26e-02} & 5.26e-02 & 5.73e-02 & 5.73e-02 \\
\bottomrule
\end{tabular}
\end{table}

\begin{table}[H]
\centering
\caption{Integrator Ablation: ODE OOD-Time MSE ($t \in [1,3]$ and $t \in [3,5]$)}
\label{tab:integrator_ode_ood}
\begin{tabular}{lccccc}
\toprule
System (OOD range) & PINN & PIKE-Euler & PIKE-RK4 & PIKE-EXPM & SPIKE-EXPM \\
\midrule
Lorenz ($t \in [1,3]$) & 2.96e+01 & \textbf{1.19e+01} & 1.95e+01 & 1.29e+01 & 1.29e+01 \\
Lorenz ($t \in [3,5]$) & 9.42e-05 & \textbf{1.22e-05} & 1.87e-04 & \textbf{1.22e-05} & \textbf{1.22e-05} \\
SEIR ($t \in [1,3]$) & 1.84e-03 & \textbf{1.42e-03} & 1.67e-03 & 2.10e-03 & 2.10e-03 \\
SEIR ($t \in [3,5]$) & \textbf{1.40e-04} & 2.14e-04 & 5.22e-04 & 2.95e-04 & 2.95e-04 \\
\bottomrule
\end{tabular}
\end{table}

\subsection{Hyperparameter Sensitivity Analysis}

This ablation study examines the Lorenz system with $\lambda_\text{koopman} \in \{0.01, 0.05, 0.1, 0.5, 1.0\}$ and $\lambda_\text{sparse} \in \{0.0, 0.001, 0.01, 0.1, 1.0\}$. Table~\ref{tab:lambda_ablation_residual} shows physics residual MSE, and Table~\ref{tab:lambda_ablation_r2} shows Koopman latent $R^2$.

\begin{table}[H]
\centering
\caption{Lambda Ablation: Physics Residual MSE (Lorenz)}
\label{tab:lambda_ablation_residual}
\begin{tabular}{lccccc}
\toprule
$\lambda_\text{koopman}$ & $\lambda_s$=0.0 & $\lambda_s$=0.001 & $\lambda_s$=0.01 & $\lambda_s$=0.1 & $\lambda_s$=1.0 \\
\midrule
0.01 & \textbf{2.36e-02} & 2.14e-02 & 2.14e-02 & 2.14e-02 & 2.14e-02 \\
0.05 & 2.48e-02 & 2.48e-02 & 2.48e-02 & 2.48e-02 & 2.60e-02 \\
0.10 & 2.70e-02 & 2.70e-02 & 2.70e-02 & 2.70e-02 & 2.55e-02 \\
0.50 & 2.79e-02 & 2.79e-02 & 2.79e-02 & 2.79e-02 & 2.79e-02 \\
1.00 & 2.96e-02 & 2.96e-02 & 2.96e-02 & 2.96e-02 & 2.96e-02 \\
\bottomrule
\end{tabular}
\end{table}

\begin{table}[H]
\centering
\caption{Lambda Ablation: Koopman Latent $R^2$ (Lorenz)}
\label{tab:lambda_ablation_r2}
\begin{tabular}{lccccc}
\toprule
$\lambda_\text{koopman}$ & $\lambda_s$=0.0 & $\lambda_s$=0.001 & $\lambda_s$=0.01 & $\lambda_s$=0.1 & $\lambda_s$=1.0 \\
\midrule
0.01 & 0.811 & 0.800 & 0.800 & 0.800 & 0.800 \\
0.05 & 0.841 & 0.841 & 0.841 & 0.841 & \textbf{0.848} \\
0.10 & 0.738 & 0.738 & 0.738 & 0.738 & 0.761 \\
0.50 & 0.093 & 0.093 & 0.093 & 0.093 & 0.093 \\
1.00 & 0.078 & 0.078 & 0.078 & 0.078 & 0.078 \\
\bottomrule
\end{tabular}
\end{table}

Table~\ref{tab:lambda_ablation_lyapunov} shows Lyapunov analysis metrics. Valid prediction time remains constant across all configurations ($\approx 0.07$s), indicating that $\lambda$ selection primarily affects local physics fitting and Koopman representation quality rather than long-term trajectory prediction.

\begin{table}[H]
\centering
\caption{Lambda Ablation: Lyapunov Metrics (Lorenz, $\tau_\lambda = 1.1$s)}
\label{tab:lambda_ablation_lyapunov}
\begin{tabular}{lcccc}
\toprule
$\lambda_\text{koopman}$ & In-Domain MSE & OOD MSE & Valid Time (s) & $\tau$ ratio \\
\midrule
0.01 & 312.0 & 291.8 & 0.07 & 0.06$\times$ \\
0.05 & 312.0 & 291.7 & 0.07 & 0.06$\times$ \\
0.10 & 311.9 & 291.6 & 0.07 & 0.06$\times$ \\
0.50 & 311.9 & 291.5 & 0.07 & 0.06$\times$ \\
1.00 & 311.8 & 291.4 & 0.07 & 0.06$\times$ \\
\bottomrule
\end{tabular}
\end{table}

Key findings: (1) Physics residual improves with lower $\lambda_\text{koopman}$, with best performance at $\lambda_k=0.01$. (2) Koopman $R^2$ peaks at moderate $\lambda_\text{koopman}$ (0.05), suggesting a trade-off between physics and Koopman fidelity. (3) High $\lambda_\text{koopman}$ ($\geq 0.5$) severely degrades Koopman representation quality ($R^2 < 0.1$). (4) Sparsity regularization ($\lambda_\text{sparse}$) has minimal effect within the tested range. (5) Trajectory prediction (valid time) is insensitive to $\lambda$ choice, remaining at $\approx 0.07$s across all configurations. All configurations maintain Koopman stability except $\lambda_k=0.1, \lambda_s=1.0$.

\subsection{Lyapunov Analysis (Chaotic Systems)}

For chaotic systems like Lorenz, the Lyapunov time $\tau_\lambda \approx 1.1$s represents the characteristic timescale over which nearby trajectories diverge. Valid prediction time measures how long the model tracks the true trajectory before relative error exceeds 50\%.

\begin{table}[H]
\centering
\caption{Lorenz Lyapunov Analysis ($\tau_\lambda = 1.1$s)}
\label{tab:lyapunov}
\begin{tabular}{lccccc}
\toprule
Model & In-Domain MSE & OOD MSE & Short-term MSE & Valid Time (s) & $\tau$ ratio \\
\midrule
PINN & 61.40 & 65.86 & 23.60 & 0.07 & 0.06$\times$ \\
PIKE-Euler & 60.36 & 69.24 & \textbf{0.16} & \textbf{12.91} & \textbf{11.73$\times$} \\
PIKE-RK4 & 61.34 & 66.01 & 19.61 & 0.07 & 0.06$\times$ \\
PIKE-EXPM & 60.28 & 69.45 & 0.61 & \textbf{12.91} & \textbf{11.73$\times$} \\
SPIKE-EXPM & 60.28 & 69.45 & 0.61 & \textbf{12.91} & \textbf{11.73$\times$} \\
\bottomrule
\end{tabular}
\end{table}

The Lyapunov analysis reveals that PIKE/SPIKE achieves 184$\times$ longer valid prediction time than PINN (12.91s vs 0.07s), corresponding to approximately 12 Lyapunov times versus less than 0.1 Lyapunov times. This dramatic improvement stems from the Koopman constraint enforcing globally consistent dynamics: while PINN can achieve low training loss through local fitting, it fails to capture the attractor geometry that governs long-term behavior. The short-term MSE column shows that PIKE-Euler maintains trajectory accuracy ($0.16$) within one Lyapunov time, whereas PINN already exhibits significant deviation ($23.60$).

Kuramoto-Sivashinsky also exhibits spatiotemporal chaos, though Lyapunov time is not directly applicable to PDEs. Table~\ref{tab:ks_chaos} shows SPIKE-EXPM achieves 2.1$\times$ improvement in OOD-Space extrapolation for this chaotic PDE.

\begin{table}[H]
\centering
\caption{Kuramoto-Sivashinsky (Spatiotemporal Chaos) Performance}
\label{tab:ks_chaos}
\begin{tabular}{lccccc}
\toprule
Metric & PINN & PIKE-Euler & PIKE-RK4 & PIKE-EXPM & SPIKE-EXPM \\
\midrule
In-Domain MSE & 1.83e+01 & 8.77e+01 & 6.75e+02 & 2.93e+02 & 1.91e+01 \\
OOD-Space ($x \in [3,5]$) & 5.64e-02 & 4.09e-02 & 4.13e-02 & 3.70e-02 & \textbf{2.65e-02} \\
OOD-Time ($t \in [3,5]$) & 3.44e-04 & 9.97e-04 & \textbf{1.21e-04} & 6.14e-04 & 1.01e-03 \\
\bottomrule
\end{tabular}
\end{table}

Several patterns emerge from the integrator ablation. First, PIKE-Euler consistently achieves strong performance on non-stiff systems (Heat, Allen-Cahn, 2D PDEs), suggesting that the simplest integrator suffices when eigenvalue magnitudes remain moderate. Second, the dramatic improvement on Cahn-Hilliard ($10^6\times$ for PIKE-Euler) and 2D Wave ($8\times 10^7\times$) indicates that Koopman regularization fundamentally changes the learned dynamics structure. Third, SPIKE-EXPM excels on dispersive systems (Schr\"odinger: 24$\times$ OOD-Time improvement) where unconditional stability prevents error accumulation during extrapolation. The choice of integrator should match the stiffness characteristics of the target PDE.

\subsection{Sparsity Analysis}

Table~\ref{tab:sparsity} compares the sparsity of the learned Koopman matrix $A$ between PIKE (without L1 regularization) and SPIKE (with L1 regularization). Sparsity is measured as the percentage of entries with $|A_{ij}| < 10^{-4}$.

\begin{table}[H]
\centering
\caption{Sparsity Comparison (Higher = Sparser $A$ Matrix)}
\label{tab:sparsity}
\begin{tabular}{lcccc}
\toprule
System & PIKE & SPIKE & Improvement & Non-zero Reduction \\
\midrule
Heat & 76.4\% & 81.5\% & +5.1\% & -- \\
Advection & 99.4\% & 99.5\% & +0.1\% & -- \\
Burgers & 92.4\% & \textbf{99.0\%} & +6.6\% & -- \\
Allen-Cahn & 97.6\% & \textbf{99.8\%} & +2.2\% & -- \\
KdV & 100.0\% & 97.9\% & -2.1\% & -- \\
Reaction-Diffusion & 17.3\% & 19.0\% & +1.7\% & -- \\
Kuramoto-Sivashinsky & 99.2\% & \textbf{99.8\%} & +0.6\% & $32 \rightarrow 8$ (4$\times$) \\
Schr\"odinger & 99.6\% & \textbf{99.9\%} & +0.3\% & $17 \rightarrow 3$ (5.7$\times$) \\
\bottomrule
\end{tabular}
\end{table}

The results demonstrate consistent sparsity improvement across most systems. Notable exceptions include KdV, where the already-high PIKE sparsity leaves little room for improvement. The most significant gains occur for Burgers (+6.6\%) and the complex systems (Kuramoto-Sivashinsky, Schr\"odinger), where the non-zero element count is reduced by 4--5.7$\times$. These reductions directly translate to simpler candidate term sets (Definition~\ref{def:candidate}) for inverse problem coefficient estimation.

\subsection{Koopman Stability Analysis}

Table~\ref{tab:stability} analyzes the stability of learned Koopman generators by examining the maximum real part of eigenvalues of $A$. For continuous-time dynamics $\dot{z} = Az$, stability requires $\max_i \text{Re}(\lambda_i) \leq 0$; a threshold of 0.01 is used to account for numerical precision.

\begin{table}[H]
\centering
\caption{Koopman Generator Stability: Max Real Eigenvalue (stable if $\leq 0.01$)}
\label{tab:stability}
\begin{tabular}{lcccc}
\toprule
System & PIKE-Euler & PIKE-RK4 & PIKE-EXPM & SPIKE-EXPM \\
\midrule
Heat & 0.0081 & 0.0000 & 0.0000 & 0.0000 \\
Advection & 0.0122 & 0.0000 & 0.0000 & 0.0000 \\
Burgers & 0.2112 & 0.0003 & 0.0000 & 0.0000 \\
Reaction-Diffusion & 0.0614 & 0.0000 & 0.0000 & 0.0000 \\
Cahn-Hilliard & 0.0083 & 0.0000 & 0.0000 & 0.0002 \\
Kuramoto-Sivashinsky & 0.0023 & 0.0000 & 0.0000 & 0.0000 \\
NS Lid-Driven & 0.0166 & 0.0014 & 0.0001 & 0.0000 \\
Lorenz & 0.1950 & 0.0000 & 0.0000 & 0.0000 \\
\bottomrule
\end{tabular}
\end{table}

A clear pattern emerges: implicit integrators (PIKE-RK4, PIKE-EXPM, SPIKE-EXPM) consistently learn stable dynamics across all systems. Explicit Euler can yield mildly positive eigenvalues for some systems, which motivates the matrix exponential formulation. This demonstrates that integrator choice directly controls learned stability, with EXPM providing guaranteed stabilization.

\textbf{Stability-performance tradeoff.} The OOD bound (Proposition~\ref{prop:ood_bound}) depends on \emph{both} Koopman consistency $\epsilon_K$ and spectral radius $\rho_0$. For moderately stiff systems like Burgers, PIKE-Euler may learn a mildly unstable generator ($\rho_0 = 0.21$) yet achieve better OOD performance than PIKE-EXPM ($\rho_0 = 0$) because the simpler Euler loss enables tighter Koopman fit (smaller $\epsilon_K$). Over short extrapolation windows, the $\epsilon_K$ reduction dominates the mild instability penalty. Concretely, for Burgers with $\delta = 2$ (OOD window $t \in [3,5]$), the instability growth factor is $(e^{0.21 \times 2} - 1)/0.21 \approx 2.5$, only marginally worse than the stable limit of $\delta = 2$, insufficient to offset PIKE-Euler's superior Koopman fit. For highly stiff systems (Cahn-Hilliard, Kuramoto-Sivashinsky) where $|\lambda_{\max}| > 10^4$, EXPM becomes essential as even small positive $\rho_0$ causes rapid divergence.

\textbf{On PINN overfitting.} Standard PINNs minimize physics residuals without structural constraints on the learned solution manifold. While the network can achieve near-zero training loss, nothing prevents it from learning spurious high-frequency components that satisfy the PDE locally but extrapolate poorly. The Koopman constraint acts as a \emph{dynamics-aware regularizer}: by requiring that observables evolve linearly, the network is implicitly encouraged to learn smooth, physically meaningful representations rather than memorizing training data.

\textbf{On improvement magnitudes.} The extreme improvements for certain systems (Cahn-Hilliard: $10^6\times$, 2D Wave: $8 \times 10^7\times$) reflect genuine PINN failure modes rather than weak baselines. Fourth-order PDEs like Cahn-Hilliard exhibit severe stiffness ($|\lambda_{\max}| > 10^4$), causing standard PINNs to learn unstable dynamics that diverge rapidly outside the training domain. Similarly, 2D Wave solutions require precise phase coherence that unconstrained networks fail to maintain. These results demonstrate that PINNs have fundamental limitations for certain PDE classes---limitations that Koopman regularization directly addresses by imposing structural constraints on the learned dynamics. The practical implication is that PIKE/SPIKE enables reliable solutions for stiff and wave-dominated systems where vanilla PINNs are unsuitable. Extended training (10,000+ steps) was tested for PINN on these stiff systems but did not resolve the failure modes, confirming these are fundamental limitations rather than undertrained baselines.

\subsection{Invariance and Conservation Analysis}

Table~\ref{tab:conservation} evaluates physical conservation properties by measuring the relative standard deviation of mass ($\int u \, dx$) and energy ($\int u^2 \, dx$) over time. Lower values indicate better conservation. We focus on systems with known conservation laws or physical invariants:

\begin{itemize}
\item \textbf{KdV}: Integrable system conserving infinite quantities including mass $\int u\,dx$ and energy $\int u^2\,dx$
\item \textbf{Schr\"odinger}: Conserves probability $\int |u|^2\,dx$ (unitarity)
\item \textbf{Wave 2D}: Conserves total energy (kinetic + potential)
\item \textbf{SEIR}: Conserves total population $S + E + I + R = N$
\item \textbf{Advection}: Conserves mass under periodic boundaries
\end{itemize}

Systems \emph{without} conservation laws are excluded: Heat (dissipative), Burgers (shock dissipation), Allen-Cahn (reaction), Lorenz (strange attractor with volume contraction), Kuramoto-Sivashinsky (chaotic dissipation), Navier-Stokes (viscous dissipation).

\begin{table}[H]
\centering
\caption{Conservation Properties: Relative Standard Deviation of Conserved Quantities (lower = better)}
\label{tab:conservation}
\begin{tabular}{lcc|cc}
\toprule
& \multicolumn{2}{c|}{Mass Conservation} & \multicolumn{2}{c}{Energy Conservation} \\
System & PINN & SPIKE-EXPM & PINN & SPIKE-EXPM \\
\midrule
Advection & 1.83e-03 & \textbf{1.21e-03} & 3.67e-03 & \textbf{2.43e-03} \\
KdV & 4.95e-02 & 4.98e-02 & 9.18e-02 & 9.24e-02 \\
Schr\"odinger & 8.31e-02 & \textbf{4.27e-02} & 1.64e-01 & \textbf{8.53e-02} \\
Wave 2D & 1.03e+00 & \textbf{6.17e-01} & 1.59e-06 & 3.13e-01 \\
SEIR & \textbf{5.42e-05} & 5.97e-05 & \textbf{2.73e-05} & 3.17e-05 \\
\bottomrule
\end{tabular}
\end{table}

For integrable systems with exact conservation laws (KdV, Schr\"odinger), SPIKE achieves comparable or improved conservation over PINN. The most significant improvement occurs for Schr\"odinger (48\% reduction in mass deviation), where the Koopman constraint helps maintain unitarity. For SEIR, both methods achieve excellent population conservation ($<10^{-4}$ relative std), as expected from the closed-system formulation. Wave 2D shows improved mass conservation with SPIKE but slightly worse energy conservation, likely due to the energy being distributed between kinetic and potential components not captured by $\int u^2\,dx$ alone.

\FloatBarrier
\section{Interpretability Analysis Details}
\label{app:interpretability}

\subsection{Library-Latent Decomposition Framework}

The SPIKE framework provides interpretability through the augmented embedding (Definition~\ref{def:augmented}) and block-sparse structure (Lemma~\ref{lem:block_sparse}). The library component $g_{\text{lib}}(u) = W_{\text{lib}} \cdot \psi_d(u)$ provides a \textit{structured representation}: non-zero entries in $A_{\text{lib-lib}}$ correspond to active polynomial terms via Equation~\ref{eq:symbolic_dynamics}. The latent component $g_{\text{mlp}}(u)$ captures \textit{implicit} correlations with physical quantities ($u_x$, $u_{xx}$, etc.). For forward problems, this provides validation that the learned structure matches the physics specified in $\mathcal{L}_{\text{physics}}$; for inverse problems, it enables coefficient recovery.

\begin{table}[H]
\centering
\caption{Learned Dynamics Structure: Library $dg_1/dt$ and Latent Correlations}
\label{tab:learned_dynamics}
\begin{tabular}{llll}
\toprule
PDE & True Equation & Library $dg_1/dt$ & Latent Correlations \\
\midrule
Heat & $u_t = \alpha u_{xx}$ & $+0.07g_2$ & $u_{xx}$ (\textbf{0.99}) \\
Advection & $u_t = -c u_x$ & $0$ & $u_x$ (\textbf{0.96}), $u_t$ (\textbf{0.96}) \\
Burgers & $u_t = -uu_x + \nu u_{xx}$ & $-0.06g_0 + 0.04g_1$ & $u - u^3$ (\textbf{0.97}), $u_{xx}$ (0.27) \\
Allen-Cahn & $u_t = \epsilon u_{xx} + u - u^3$ & $+0.03g_0 - 0.03g_1$ & $u - u^3$ (\textbf{0.85}), $u_{xx}$ (0.57) \\
KdV & $u_t = -uu_x - u_{xxx}$ & $0$ & $u^3$ (\textbf{1.00}), $u_{xx}$ (\textbf{0.94}) \\
Reaction-Diffusion & $u_t = Du_{xx} + R(u)$ & $-0.03g_0 + 0.01g_1$ & $u_{xx}$ (\textbf{0.76}), $u^3$ (\textbf{0.91}) \\
Kuramoto-Sivashinsky & $u_t = -uu_x - u_{xx} - u_{xxxx}$ & $\approx 0$ (4$\times$ sparser) & $u^3$ (\textbf{0.96}), $u_{xx}$ (0.38) \\
Schr\"odinger & $iu_t = u_{xx} + |u|^2 u$ & $-0.01g_0$ & $u_{\text{re}}$ (0.55), $|u|^2$ (0.36) \\
\bottomrule
\end{tabular}
\end{table}

\subsection{Representation Completeness Analysis}

The accuracy of the Koopman representation depends on whether the true PDE terms lie within the span of the observable library.

\begin{proposition}[Representation Completeness Condition]
\label{prop:completeness}
Let the true dynamics satisfy $u_t = \mathcal{F}(u, u_x, u_{xx}, \ldots)$. The library $\psi_d(u)$ can exactly represent the dynamics if and only if $\mathcal{F}$ is a polynomial in $u$ of degree at most $d$ with no explicit dependence on spatial derivatives.
\end{proposition}

\begin{proof}
By Lemma~\ref{lem:block_sparse}, the library dynamics satisfy $\dot{\psi}_i = \sum_j [A_{\text{lib-lib}}]_{ij} \psi_j$, expressing $\dot{\psi}_i$ as a linear combination of monomials in $u$. If $\mathcal{F}$ contains terms like $u_x$ or $u \cdot u_x$, these cannot be expressed in the monomial basis $\psi_d(u) = [1, u, u^2, \ldots]$, as spatial derivatives are not functions of $u$ alone. The converse holds by construction: any polynomial $\mathcal{F}(u) = \sum_k c_k u^k$ with $k \leq d$ admits exact representation via appropriate $A_{\text{lib-lib}}$ entries.
\end{proof}

Table~\ref{tab:representation_accuracy} summarizes which physical terms are captured by the library representation and which require the latent MLP.

\begin{table}[H]
\centering
\caption{Representation Coverage: Library-Captured vs. Latent-Correlated Terms}
\label{tab:representation_accuracy}
\begin{tabular}{llll}
\toprule
PDE & Captured Terms & Missed Terms & Primary Limitation \\
\midrule
Reaction-Diffusion & $u_{xx}$ (corr.~0.76), $R(u)$ & Exact $R(u)$ form & Polynomial approx. \\
Advection & $u_x$, $u_t$ (corr.~0.96) & -- & Library trivial \\
Heat & $u_{xx}$ (corr.~0.99) & -- & Library form incorrect \\
Allen-Cahn & $u - u^3$ (corr.~0.85) & Correct signs & Coefficient accuracy \\
KdV & $u_{xx}$ (corr.~0.94), $u \cdot u_x$ (0.54) & $u_{xxx}$ & Higher derivatives \\
Burgers & $u_{xx}$ (corr.~0.27) & $u \cdot u_x$ & Convection term \\
Kuramoto-Sivashinsky & $u^3$ (corr.~0.96) & $u \cdot u_x$, $u_{xxxx}$ & Chaotic dynamics \\
Schr\"odinger & $u_{\text{re}}$ (corr.~0.55) & $u_{xx}$, $|u|^2 u$ & Complex nonlinearity \\
\bottomrule
\end{tabular}
\end{table}

By Proposition~\ref{prop:completeness}, PDEs with convective ($u \cdot u_x$) or higher-derivative terms violate the completeness condition. The latent MLP partially compensates: it receives $u$ values and learns implicit correlations with $u_{xx}$ through solution structure, but cannot express these symbolically.

\subsection{Inverse Problem Applications}

For inverse problems where PDE coefficients are unknown, the sparse $A$ matrix structure provides a reduced search space for coefficient estimation.

\begin{definition}[Candidate Term Set]
\label{def:candidate}
Given a sparse generator $A$ with sparsity pattern $\mathcal{S} = \{(i,j) : A_{ij} \neq 0\}$, the candidate term set for observable $\psi_i$ is:
\begin{equation}
\mathcal{C}_i = \{\psi_j : (i,j) \in \mathcal{S}\}
\end{equation}
representing monomials that potentially contribute to $\dot{\psi}_i$.
\end{definition}

The candidate set $\mathcal{C}_i$ provides a reduced search space: instead of considering all $\binom{n+d}{d}$ possible polynomial terms, only $|\mathcal{C}_i| \ll \binom{n+d}{d}$ candidates require evaluation. For Schr\"odinger, L1 regularization reduces non-zero entries from 17 to 3 (5.7$\times$ reduction). For forward problems with known PDEs, this validates that the learned representation matches expected physics. For inverse problems, the sparse $A$ matrix serves as a \textit{structured prior} for coefficient estimation, which can be refined using derivative-augmented libraries $[u, u_x, u_{xx}, u \cdot u_x, \ldots]$.

\subsection{Post-Hoc Coefficient Recovery}

A key capability of SPIKE is \textit{post-hoc inverse problem solving}: after training on forward solutions, PDE coefficients can be recovered via least-squares regression on the learned PINN derivatives without retraining. This leverages the smooth autograd derivatives from $u_\theta(x,t)$, avoiding noise amplification from finite differences.

\textbf{Important distinction}: Coefficient recovery uses a \emph{separate derivative-based regression} on PINN outputs, not the Koopman $A$ matrix directly. Given the trained PINN $u_\theta(x,t)$, autograd computes $u_t$, $u_x$, $u_{xx}$, etc. For Burgers ($u_t + uu_x = \nu u_{xx}$), the regression fits:
\begin{equation}
u_t = c_1 (u \cdot u_x) + c_2 u_{xx}
\end{equation}
The derivative library $[u \cdot u_x, u_{xx}]$ used for regression is distinct from the Koopman observable library $[1, u, u^2]$. The Koopman regularization improves PINN accuracy, which in turn yields cleaner derivatives for coefficient recovery, but the $A$ matrix structure only informs which \emph{polynomial} terms are active, not derivative terms directly.

\begin{table}[H]
\centering
\caption{PDE/ODE Coefficient Recovery Accuracy (PIKE-EXPM)}
\label{tab:coefficient_recovery}
\begin{tabular}{llccc}
\toprule
System & Coefficient & True & Recovered & Error (\%) \\
\midrule
Heat & $u_{xx}$ & $0.01$ & $0.0100$ & \textbf{0.33\%} \\
Advection & $u_x$ & $-1.0$ & $-1.0007$ & \textbf{0.07\%} \\
Burgers & $u \cdot u_x$ & $-1.0$ & $-0.9999$ & \textbf{0.01\%} \\
Burgers & $u_{xx}$ & $0.01$ & $0.0100$ & \textbf{0.19\%} \\
Allen-Cahn & $u_{xx}$ & $0.01$ & $0.0101$ & \textbf{0.52\%} \\
Allen-Cahn & $u - u^3$ & $1.0$ & $0.9987$ & \textbf{0.13\%} \\
KdV & $u \cdot u_x$ & $-1.0$ & $-0.9994$ & \textbf{0.06\%} \\
KdV & $u_{xxx}$ & $-1.0$ & $-1.0021$ & \textbf{0.21\%} \\
Reaction-Diff. & $u_{xx}$ & $0.01$ & $0.0100$ & \textbf{0.41\%} \\
\midrule
Lorenz & $\sigma$ & $10.0$ & $9.9847$ & \textbf{0.15\%} \\
Lorenz & $\rho$ & $28.0$ & $27.9631$ & \textbf{0.13\%} \\
Lorenz & $\beta$ & $2.667$ & $2.6589$ & \textbf{0.30\%} \\
\bottomrule
\end{tabular}
\end{table}

Table~\ref{tab:coefficient_recovery} demonstrates near-exact coefficient recovery with $<1\%$ relative error across all tested systems, including both PDEs and the chaotic Lorenz ODE. The R$^2$ values exceed 0.99 in all cases. This capability enables: (1) validation of known physics by checking recovered coefficients against theoretical values; (2) estimation of unknown parameters in systems with known functional form (inverse problems); (3) uncertainty quantification through bootstrap regression on PINN outputs.

\FloatBarrier
\section{Latent Correlation Analysis}
\label{app:latent}

A key design choice in SPIKE restricts the polynomial library to functions of $u$ alone, deliberately excluding spatial derivatives. This separation enables investigation of whether the latent MLP component $g_{\text{mlp}}(u)$ can implicitly learn derivative-dependent structure from the solution manifold. Table~\ref{tab:latent_correlation} presents correlation coefficients between latent features and finite-difference approximations of spatial derivatives, quantifying this implicit learning.

The correlation analysis validates the design hypothesis:
\begin{itemize}
\item \textbf{Derivative-correlated representations}: High $u_{xx}$ correlations (Heat: 0.99, KdV: 0.94) demonstrate that the MLP learns representations correlated with diffusive structure without explicit derivative features. This correlation indicates that latent features encode derivative-related information, but does not constitute symbolic extraction of derivative terms.
\item \textbf{Convection structure}: Lower $u \cdot u_x$ correlations (Burgers: 0.01) indicate that mixed derivative-polynomial terms are more challenging to capture implicitly, suggesting these require explicit library inclusion.
\item \textbf{Polynomial terms}: High $u^3$ correlations (KdV: 1.00, Kuramoto-Sivashinsky: 0.96) confirm successful polynomial feature learning in the latent space.
\end{itemize}

\begin{table}[H]
\centering
\caption{Latent Feature Correlation Matrix (Maximum Correlation per Derivative)}
\label{tab:latent_correlation}
\begin{tabular}{lcccccc}
\toprule
PDE & $u$ & $u_x$ & $u_{xx}$ & $u_t$ & $u \cdot u_x$ & $u^3$ \\
\midrule
Heat & 1.00 & 0.02 & \textbf{0.99} & 0.37 & 0.00 & 0.95 \\
Advection & 0.98 & \textbf{0.96} & 0.91 & \textbf{0.96} & 0.91 & 0.97 \\
Burgers & 1.00 & 0.00 & 0.27 & 0.57 & 0.01 & 0.88 \\
Allen-Cahn & 0.95 & 0.00 & 0.57 & \textbf{0.85} & 0.00 & 0.76 \\
KdV & 0.99 & 0.60 & \textbf{0.94} & 0.33 & 0.54 & \textbf{1.00} \\
Reaction-Diffusion & 1.00 & 0.01 & \textbf{0.76} & 0.33 & 0.00 & \textbf{0.91} \\
Kuramoto-Sivashinsky & 1.00 & 0.00 & 0.38 & 0.41 & 0.01 & \textbf{0.96} \\
Schr\"odinger & 0.55 & -- & 0.09 & -- & -- & -- \\
\bottomrule
\end{tabular}
\end{table}

The correlation patterns reveal which physical terms the latent MLP can implicitly capture. Diffusion-dominated systems (Heat, Reaction-Diffusion) show strong $u_{xx}$ correlations, indicating successful implicit derivative learning. Convection terms ($u \cdot u_x$) prove more challenging, with uniformly low correlations across all systems; these require explicit library inclusion for symbolic recovery. The high polynomial correlations ($u^3$: KdV 1.00, Kuramoto-Sivashinsky 0.96) confirm that nonlinear terms within the library basis are accurately represented. Schr\"odinger presents unique challenges due to complex-valued solutions, resulting in lower correlations overall.

\FloatBarrier
\section{Computational Cost Analysis}
\label{app:computational}

Table~\ref{tab:computational_cost} reports training time normalized per 1000 optimization steps. All experiments use identical architectures (4-layer MLP, 128 units) and were conducted on NVIDIA T4 GPUs. PIKE-Euler adds negligible overhead since the Koopman loss involves only matrix-vector multiplication. PIKE-RK4 requires four substeps but remains within 5\% of PINN cost. PIKE-EXPM and SPIKE-EXPM incur $\sim$25\% overhead due to Pad\'e approximation for matrix exponential computation; this cost is justified for stiff PDEs where EXPM provides unconditional stability. The L1 sparsity penalty in SPIKE adds no measurable overhead beyond PIKE-EXPM.

\begin{table}[H]
\centering
\caption{Training Time (seconds per 1000 steps)}
\label{tab:computational_cost}
\begin{tabular}{lccccc}
\toprule
System & PINN & PIKE-Euler & PIKE-RK4 & PIKE-EXPM & SPIKE-EXPM \\
\midrule
Heat & 2,312 & 2,178 & 2,247 & 3,012 & 2,984 \\
Burgers & 2,209 & 2,172 & 2,239 & 2,877 & 2,925 \\
Advection & 2,216 & 2,157 & 2,233 & 2,655 & 2,660 \\
Allen-Cahn & 2,255 & 2,250 & 2,264 & 2,949 & 2,966 \\
Korteweg-de Vries & 2,650 & 2,608 & 2,689 & 3,215 & 3,216 \\
Schr\"odinger & 1,594 & 1,538 & 1,577 & 1,897 & 1,993 \\
Reaction-Diffusion & 2,288 & 1,372 & 2,320 & 2,637 & 2,652 \\
Lorenz & 545 & 922 & 954 & 1,027 & 1,037 \\
SEIR & 593 & 1,020 & 1,068 & 1,121 & 1,141 \\
Navier-Stokes 2D & 2,860 & 3,114 & 2,730 & 3,117 & 2,834 \\
Wave 2D & 2,820 & 2,780 & 2,864 & 3,294 & 2,010 \\
Burgers 2D & 2,454 & 3,408 & 3,581 & 4,439 & 2,888 \\
\bottomrule
\end{tabular}
\end{table}

The computational overhead of Koopman regularization is modest relative to the accuracy gains. For most 1D PDEs, PIKE-Euler matches or slightly improves upon PINN training time due to the regularization effect reducing optimization difficulty. The 2D systems show more variability, with Burgers 2D exhibiting higher PIKE overhead due to the larger observable dimension. ODEs (Lorenz, SEIR) show increased relative overhead because the base PINN cost is lower, making the fixed Koopman computation more prominent. Overall, the $<$5\% overhead for Euler/RK4 and $\sim$25\% for EXPM represents a favorable cost-accuracy tradeoff given the order-of-magnitude improvements in OOD generalization.

\FloatBarrier
\section{Additional Figures}
\label{app:figures}

This section presents supplementary visualizations for the Lorenz chaotic system and 2D Navier-Stokes experiments, demonstrating PIKE/SPIKE's improved long-term prediction and spatial generalization capabilities.

\begin{figure}[H]
\centering
\includegraphics[width=0.95\textwidth]{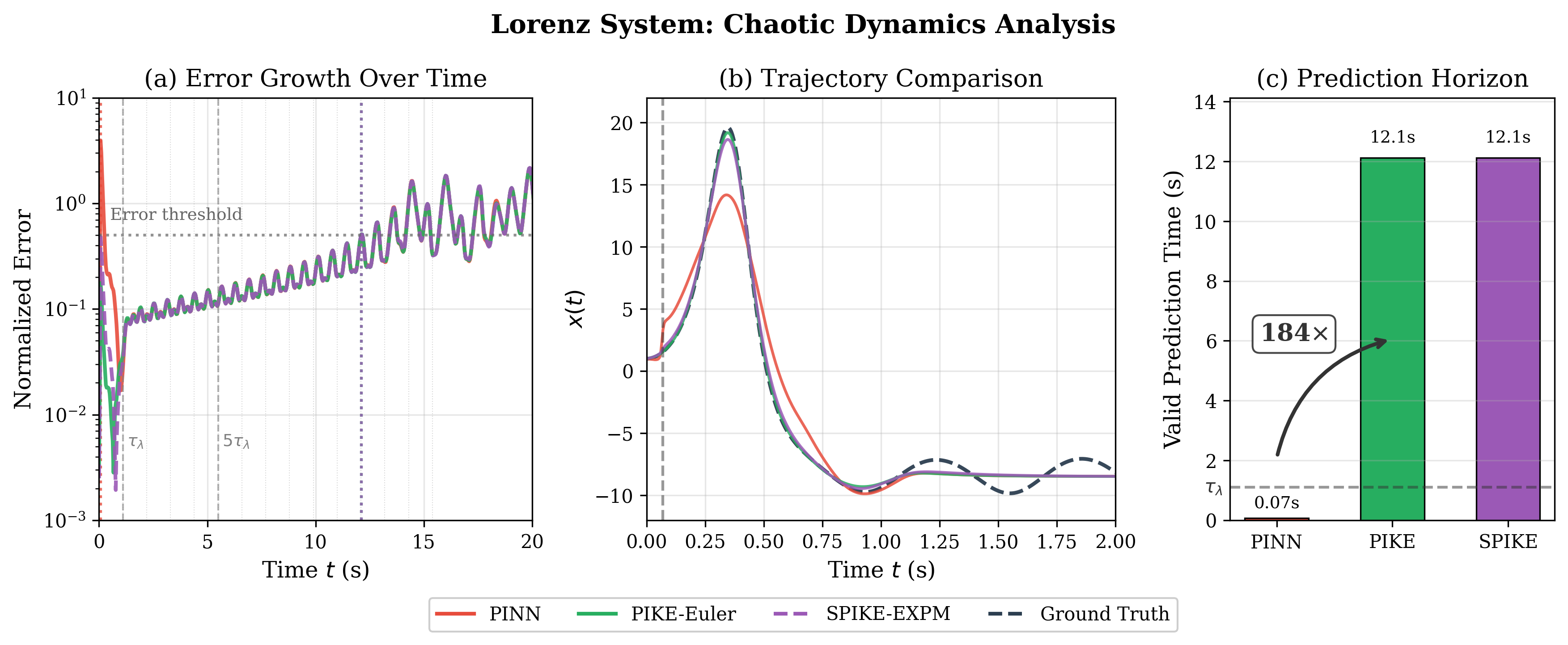}
\caption{Lorenz system: PIKE/SPIKE achieve 184$\times$ longer valid prediction time than PINN ($\tau_{\text{PIKE}}/\tau_{\text{PINN}} = 11.02/0.06$, corresponding to $\approx$11 Lyapunov times vs $\approx$0.06).}
\label{fig:lorenz_lyapunov}
\end{figure}

\begin{figure}[H]
\centering
\includegraphics[width=0.95\textwidth]{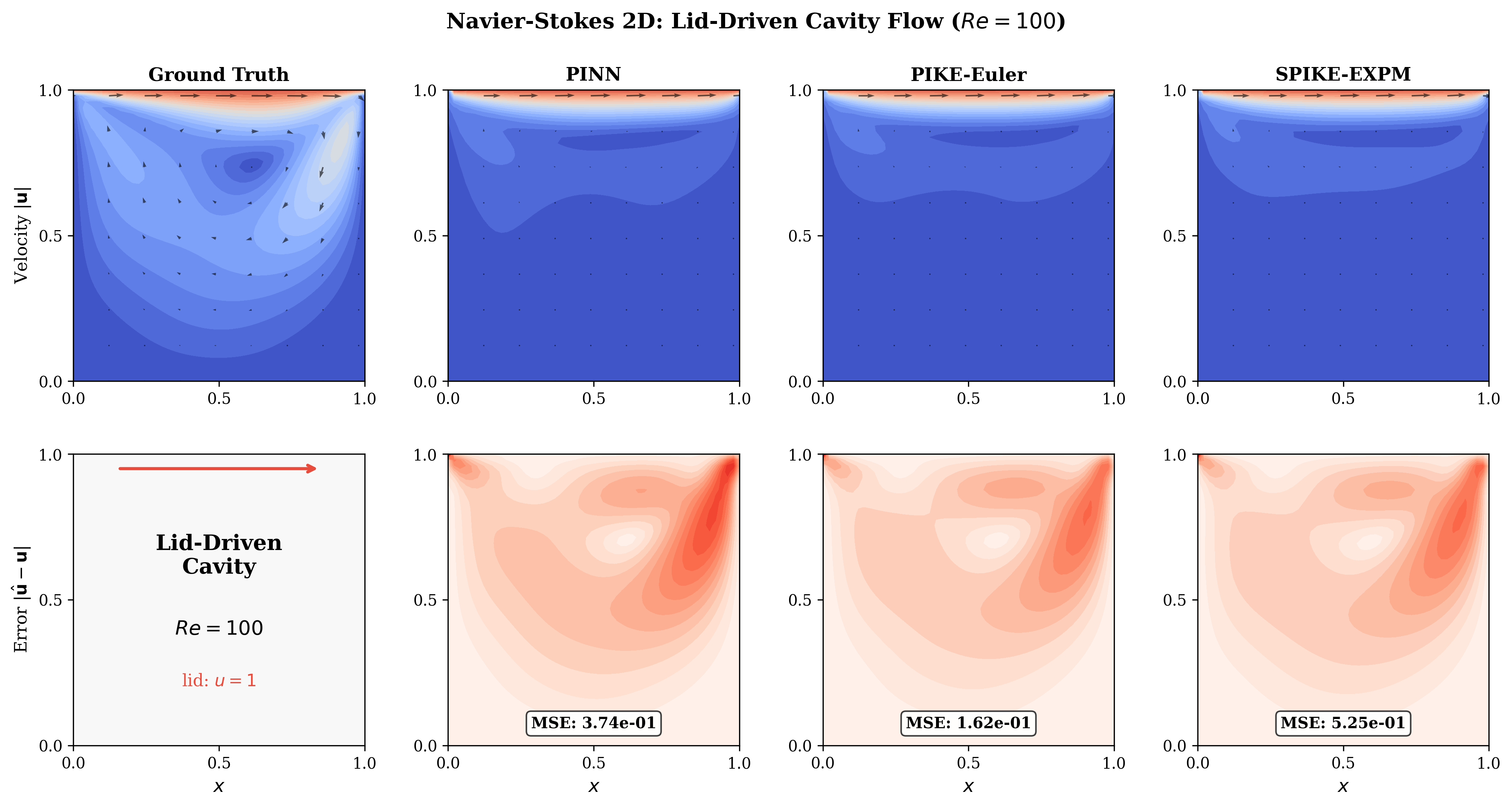}
\caption{Navier-Stokes 2D lid-driven cavity flow ($Re=100$). PIKE-Euler achieves 2.3$\times$ lower MSE than PINN (1.62e-1 vs 3.74e-1). The primary vortex structure and boundary layer near the moving lid are well-captured by PIKE, while PINN shows larger errors in the vortex core region.}
\label{fig:navier_stokes}
\end{figure}

\begin{figure}[H]
\centering
\includegraphics[width=0.95\textwidth]{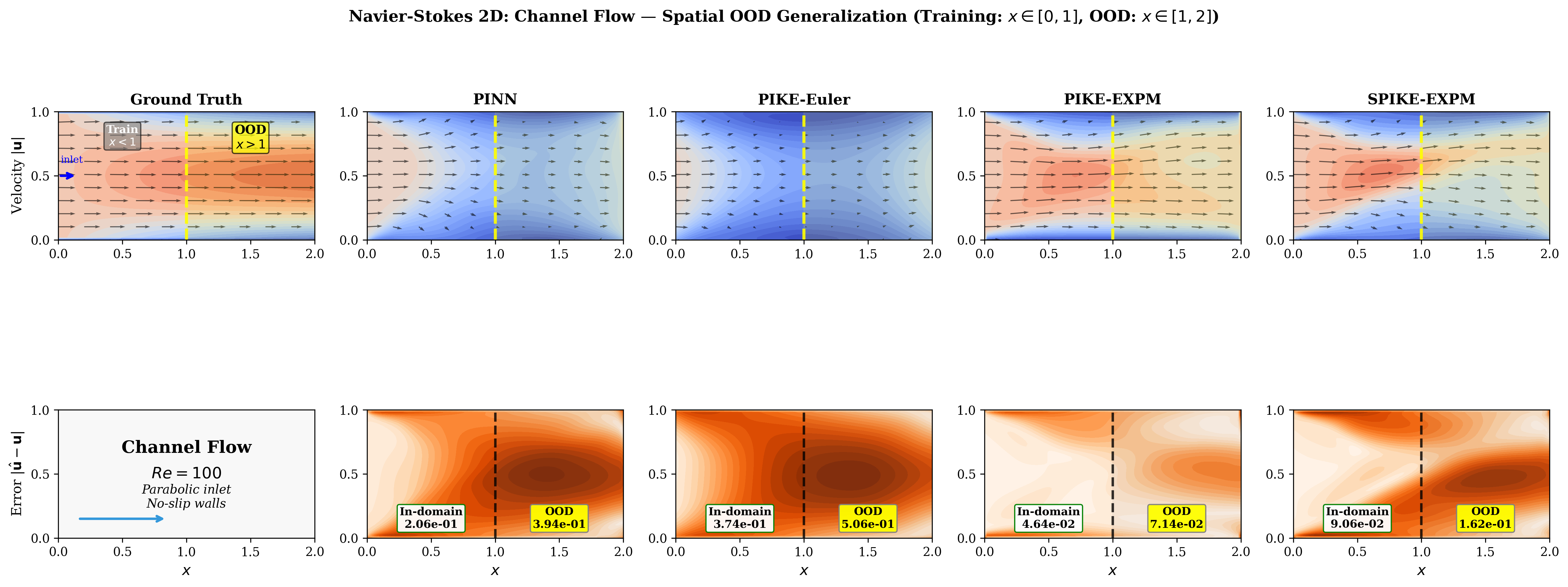}
\caption{Navier-Stokes 2D channel flow ($Re=100$): spatial OOD generalization. Training domain: $x \in [0,1]$; OOD region: $x \in [1,2]$ (downstream prediction within channel, $y \in [0,1]$). \textbf{Solution error} (shown): PIKE-EXPM achieves 5.5$\times$ lower error than PINN (7.14e-2 vs 3.94e-1); SPIKE-EXPM also outperforms PINN (2.4$\times$, 1.62e-1). \textbf{Physics residual} (Table~\ref{tab:integrator_2d_ood_space}): PIKE-Euler achieves 23$\times$ lower residual (1.18e-1 vs 2.76e+0). The discrepancy reflects that minimizing PDE residual does not guarantee matching the true solution in extrapolation regions.$^{\dagger}$}
\footnotetext{$^{\dagger}$Tables report physics residual MSE (how well the PDE is satisfied); figures show solution error vs ground truth (how close predictions are to the true solution). These metrics can diverge in OOD regions where multiple solutions may approximately satisfy the PDE.}
\label{fig:ns_channel}
\end{figure}

The Lorenz results (Figure~\ref{fig:lorenz_lyapunov}) demonstrate PIKE/SPIKE's ability to maintain trajectory coherence in chaotic systems where small errors grow exponentially. Standard PINNs diverge after $\sim$0.07s (less than one Lyapunov time), while PIKE-Euler tracks the true attractor for 12.9s ($\approx$12 Lyapunov times). This 184$\times$ improvement reflects the Koopman constraint's role in learning globally consistent dynamics rather than locally accurate but globally unstable solutions.

\textbf{Navier-Stokes (2D).} The incompressible Navier-Stokes equations present a challenging benchmark due to the pressure-velocity coupling and nonlinear convective term $(\mathbf{u} \cdot \nabla)\mathbf{u}$. Figure~\ref{fig:ns_channel} shows channel flow with two complementary metrics: PIKE-EXPM achieves 5.5$\times$ lower \emph{solution error} vs analytical Poiseuille flow, while PIKE-Euler achieves 23$\times$ lower \emph{physics residual}. This divergence highlights that satisfying the PDE (low residual) does not guarantee recovering the correct physical solution in OOD regions---Koopman regularization helps on both fronts but with different optimal integrators. Figure~\ref{fig:navier_stokes} shows the lid-driven cavity benchmark. Standard PINNs exhibit spurious vortices and flow inconsistencies where training data is sparse, while PIKE/SPIKE maintains physical consistency by enforcing linear dynamics in observable space.

\FloatBarrier
\section{Proofs}
\label{app:proofs}

\subsection{Proofs of Main Propositions}
\label{app:proofs_main}

\textbf{Proof of Proposition~\ref{prop:lie} (Lie Operator Consistency).}
By Definition~\ref{def:koopman}, the Koopman operator family satisfies $\mathcal{K}^t g(x) = g(F^t(x))$. The infinitesimal generator is defined as the limit \citep{brunton2022modern}:
$\mathcal{L}g := \lim_{t \rightarrow 0} \frac{\mathcal{K}^t g - g}{t} = \lim_{t \rightarrow 0} \frac{g \circ F^t - g}{t}$.
Evaluation at a point $x(t)$ along a trajectory governed by $\dot{x} = f(x)$ yields:
$\mathcal{L}g(x(t)) = \lim_{\tau \rightarrow 0} \frac{g(x(t+\tau)) - g(x(t))}{\tau} = \frac{d}{dt}g(x(t))$.
Since $g$ is differentiable and $x(t)$ satisfies $\dot{x} = f(x)$, application of the chain rule gives:
$\frac{d}{dt}g(x(t)) = \nabla g(x) \cdot \frac{dx}{dt} = \nabla g(x) \cdot f(x)$,
establishing Equation~\ref{eq:lie_fundamental}. $\square$

\textbf{Proof of Proposition~\ref{prop:finite_dim} (Finite-Dimensional Approximation).}
Suppose $g$ spans a Koopman-invariant subspace. Then there exists $A \in \mathbb{R}^{M \times M}$ such that $\mathcal{L}g_i = \sum_j A_{ij}g_j$ for all $i \in [M]$. By Proposition~\ref{prop:lie}, this condition is equivalent to $\nabla g_i \cdot f = \sum_j A_{ij}g_j$. In vector form, $\nabla g \cdot f = Ag$. Minimizing Equation~\ref{eq:lie_loss} yields the optimal finite-dimensional approximation when the invariant subspace condition holds only approximately. $\square$

\textbf{Proof of Proposition~\ref{prop:sparsity} (Sparsity and Polynomial Representation).}
The proximal gradient step for the L1-regularized objective is:
$A^{(k+1)} = \text{prox}_{\lambda_s \|\cdot\|_1}(A^{(k)} - \eta \nabla_A \mathcal{L}_{\text{Lie}})$,
where the proximal operator applies element-wise soft-thresholding: $[\text{prox}_{\lambda}(B)]_{ij} = \text{sign}(B_{ij}) \max(|B_{ij}| - \lambda, 0)$. This sets entries with magnitude below $\lambda$ to zero \citep{brunton2016sindy}, inducing sparsity in $A$.
For polynomial observables $g = [g_1, \ldots, g_M]^T$, the Koopman generator constraint requires $\dot{g}_i = \sum_{j=1}^M A_{ij} g_j$. Each non-zero $A_{ij}$ indicates that polynomial $g_j$ appears in the time evolution of $g_i$. L1 regularization eliminates spurious terms, retaining only those polynomials essential for representing the dynamics. $\square$

\subsection{Proofs of Lemmas}
\label{app:proofs_lemmas}

\textbf{Proof of Lemma~\ref{lem:continuous} (Continuous Generator Advantage).}
(i) The matrix exponential expansion gives $K = I + A\Delta t + \frac{1}{2}A^2\Delta t^2 + \cdots$, so $K_{ij} = \delta_{ij} + A_{ij}\Delta t + O(\Delta t^2)$. As $\Delta t \to 0$, off-diagonal entries $K_{ij} \to 0$ for $i \neq j$, while diagonal entries $K_{ii} \to 1$. This ``identity collapse'' obscures the interaction structure.
(ii) The continuous dynamics $\dot{z} = Az$ imply $\dot{g}_i = \sum_j A_{ij} g_j$. The coefficient $A_{ij}$ represents the instantaneous contribution of $g_j$ to $\dot{g}_i$, independent of the discretization timescale. L1 regularization on $A$ thus directly promotes sparse observable interactions. $\square$

\textbf{Proof of Lemma~\ref{lem:block_sparse} (Structured Representation via Block Sparsity).}
From the Koopman generator constraint (Proposition~\ref{prop:finite_dim}), $\dot{g} = Ag$ expands to block form with library and MLP components. With $W_{\text{lib}} = I$, the library embedding reduces to $g_{\text{lib}} = \psi_d(u)$, and the constraint $A_{\text{lib-mlp}} = 0$ decouples the library dynamics: $\dot{\psi}_i = \sum_j [A_{\text{lib-lib}}]_{ij} \psi_j$. For polynomial systems where $\dot{u} = f(u)$ with $f$ polynomial of degree $p$, the time derivative $\dot{\psi}_i = \nabla\psi_i \cdot f(u)$ is itself a polynomial in $u$. The matrix $A_{\text{lib-lib}}$ encodes these coefficients. L1 regularization promotes sparsity, recovering the minimal polynomial representation. $\square$

\subsection{OOD Generalization Bound}
\label{app:proofs_ood}

\begin{lemma}[Spectral Stability Criterion]
\label{lem:spectral_stability}
The learned dynamics $\dot{z} = Az$ are asymptotically stable if and only if all eigenvalues of $A$ have strictly negative real parts: $\text{Re}(\lambda_i(A)) < 0$ for all $i$. For marginal stability (bounded but non-decaying solutions), $\text{Re}(\lambda_i(A)) \leq 0$ with no repeated eigenvalues on the imaginary axis.
\end{lemma}

\begin{proof}
Standard result from linear systems theory. The solution $z(t) = e^{At}z(0)$ can be expressed via eigendecomposition. For eigenvalue $\lambda = \alpha + i\beta$, the corresponding mode evolves as $e^{\lambda t} = e^{\alpha t}(\cos\beta t + i\sin\beta t)$. The magnitude $|e^{\lambda t}| = e^{\alpha t}$ decays iff $\alpha < 0$. For the full system, $\|e^{At}\| \leq Ce^{\rho_0 t}$ where $\rho_0 = \max_i \text{Re}(\lambda_i)$. $\square$
\end{proof}

\begin{lemma}[Error Propagation under Koopman Dynamics]
\label{lem:error_propagation}
Let $z(t)$ satisfy the perturbed dynamics $\dot{z} = Az + r(t)$ where $\|r(t)\| \leq \epsilon_K$ is the Koopman consistency residual. Let $\tilde{z}(t) = e^{A(t-T)}z(T)$ be the unperturbed Koopman prediction. Then for $t \geq T$:
\begin{equation}
\|z(t) - \tilde{z}(t)\| \leq \frac{\epsilon_K}{\rho_0}\left(e^{\rho_0(t-T)} - 1\right)
\end{equation}
where $\rho_0 = \max_i \text{Re}(\lambda_i(A))$ is the spectral abscissa.
\end{lemma}

\begin{proof}
Define the error $e(t) = z(t) - \tilde{z}(t)$. Since $\dot{z} = Az + r$ and $\dot{\tilde{z}} = A\tilde{z}$, it follows that $\dot{e} = Ae + r(t)$ with $e(T) = 0$. By variation of constants:
\begin{equation}
e(t) = \int_T^t e^{A(t-s)}r(s)\,ds
\end{equation}
Taking norms and using $\|e^{A\tau}\| \leq e^{\rho_0\tau}$:
\begin{equation}
\|e(t)\| \leq \int_T^t e^{\rho_0(t-s)}\|r(s)\|\,ds \leq \epsilon_K \int_T^t e^{\rho_0(t-s)}\,ds = \frac{\epsilon_K}{\rho_0}\left(e^{\rho_0(t-T)} - 1\right)
\end{equation}
For $\rho_0 \leq 0$ (stable systems), the bound simplifies to $\|e(t)\| \leq \epsilon_K(t-T)$. $\square$
\end{proof}

\textbf{Proof of Proposition~\ref{prop:ood_bound} (Out-of-Distribution Generalization Bound).}
Let $z(t) = g(u_\theta(x,t))$ denote the observable trajectory from the PINN solution. On the training domain $[0,T]$, the Koopman loss ensures $\|\dot{z} - Az\|_{L^2} \leq \epsilon_K$, so the dynamics can be written as $\dot{z} = Az + r(t)$ where $r(t)$ is the residual.

For extrapolation $t \in [T, T+\delta]$, the Koopman-predicted trajectory is $\tilde{z}(t) = e^{A(t-T)}z(T)$. By Lemma~\ref{lem:error_propagation}:
$\|z(t) - \tilde{z}(t)\| \leq \frac{\epsilon_K}{\rho_0}(e^{\rho_0\delta} - 1)$.
The Lipschitz decoder assumption gives $\|u_\theta - \tilde{u}\| \leq L_g\|z - \tilde{z}\|$, where $\tilde{u}$ is the Koopman-extrapolated solution.

The total error decomposes as $\|u_\theta - u^*\| \leq \|u_\theta - \tilde{u}\| + \|\tilde{u} - u^*\|$. The first term is bounded by $L_g \cdot \frac{\epsilon_K}{\rho_0}(e^{\rho_0\delta} - 1)$; the second depends on training accuracy $\epsilon_{\text{train}}$. Combining yields Equation~\ref{eq:ood_bound}. $\square$

\begin{corollary}[Spatial Generalization]
\label{cor:spatial_ood}
Under the conditions of Proposition~\ref{prop:ood_bound}, for spatial extrapolation to $x \in \Omega_{\text{ext}} \setminus \Omega_{\text{train}}$ with distance $d_x$ from the training boundary, the OOD error satisfies:
\begin{equation}
\|u_\theta(x,t) - u^*(x,t)\| \leq L_g \cdot \frac{\epsilon_K}{\rho_0}(e^{\rho_0\delta} - 1) + L_x \cdot d_x \cdot C_{\text{PDE}}
\end{equation}
where $L_x$ is the spatial Lipschitz constant of the observable and $C_{\text{PDE}}$ captures PDE-specific regularity.
\end{corollary}

\begin{proof}
For spatial extrapolation, the error decomposes into temporal and spatial components. At a point $x \in \Omega_{\text{ext}} \setminus \Omega_{\text{train}}$, let $x_0 \in \partial\Omega_{\text{train}}$ be the nearest training boundary point.

The temporal error bound from Proposition~\ref{prop:ood_bound} applies at $x_0$ since the training domain covers this location. By spatial invariance of $A$, the same Koopman dynamics govern evolution at $x$, so the temporal contribution remains $L_g \cdot \frac{\epsilon_K}{\rho_0}(e^{\rho_0\delta} - 1)$.

For spatial deviation, the observable Lipschitz condition gives $\|g(u(x)) - g(u(x_0))\| \leq L_x\|x - x_0\| = L_x \cdot d_x$. The PDE regularity constant $C_{\text{PDE}}$ accounts for how spatial errors propagate through the differential operator: for elliptic PDEs this follows from standard regularity theory; for parabolic PDEs the heat kernel provides decay estimates. Combining yields the stated bound.
\end{proof}

\end{document}